\documentclass{article}

\usepackage[final]{neurips_2022}
\usepackage{neurips_2022}
\usepackage{graphicx}
\usepackage{arydshln}

\usepackage[utf8]{inputenc} 
\usepackage[T1]{fontenc}    
\usepackage{hyperref}       
\usepackage{url}            
\usepackage{booktabs}       
\usepackage{amsfonts}       
\usepackage{nicefrac}       
\usepackage{microtype}      
\usepackage{array}
\usepackage{bbm}
\usepackage{xcolor}         
\usepackage{natbib}
\usepackage{arydshln}
\usepackage{enumitem}
\usepackage{amsmath}
\usepackage{amsthm}
\usepackage{bbm}

\usepackage[capitalise]{cleveref}

\newtheorem{theorem}{Theorem}

\DeclareMathAlphabet\mathbfcal{OMS}{cmsy}{b}{n}

\newcommand{\E}{\mathbb{E}}
\newcommand{\hist}{\mathcal{H}}
\newcommand{\V}{\mathbb{V}}

\newcommand{\Prob}{\mathbb{P}}

\newcommand{\bigO}{\mathcal{O}}
\newcommand{\prob}{\Prob}
\newcommand{\var}{\text{Var}}

\newcommand{\ind}{\mathbbm{1}}
\newcommand{\query}[1]{$\text{Q}{#1}$}
\newcommand{\bfquery}[1]{$\mathbfcal{Q}_{#1}$}

\title{Predictive Querying for Autoregressive Neural Sequence Models}

\author{%
Alex Boyd\thanks{Authors contributed equally}$^{\ \ 1}$ \quad Sam Showalter$^{* \ \! 2}$ \quad Stephan Mandt$^{1,2}$ \quad Padhraic Smyth$^{1,2}$ \\
$^1$Department of Statistics \quad $^2$Department of Computer Science \\
University of California, Irvine \\
\texttt{\{alexjb,showalte,mandt,p.smyth\}@uci.edu}\\
}

\begin{document}

\maketitle

\begin{abstract}
In reasoning about sequential events it is natural to pose probabilistic queries  such as ``when will event A occur next'' or ``what is the probability of A occurring before B'', with applications in areas such as user modeling, medicine, and finance. However, with machine learning shifting towards neural autoregressive models such as RNNs and transformers, probabilistic querying has been largely restricted to simple cases such as next-event prediction. This is in part due to the fact that  future querying involves marginalization over large path spaces, which is not straightforward to do efficiently in such  models. In this paper we introduce a general typology for predictive queries in neural autoregressive sequence models and show that such queries can be systematically represented by sets of elementary building blocks. We leverage this typology to develop new query estimation methods based on beam search, importance sampling, and hybrids. Across four large-scale sequence datasets from different application domains, as well as for the GPT-2 language model, we demonstrate the ability to make query answering tractable for arbitrary queries in exponentially-large predictive path-spaces, and find clear differences in cost-accuracy tradeoffs between search and sampling methods. 
\end{abstract}

\section{Introduction}



One of the major successes in machine learning in recent years has been the development of neural sequence models for categorical sequences, particularly in natural language applications but also in other  areas such as automatic code generation and program synthesis \citep{shin2019program,chen2021evaluating}, computer security \citep{brown2018recurrent},  recommender systems \citep{wu2017recurrent}, genomics \citep{shin2021protein,amin2021generative}, and survival analysis \citep{lee2019dynamic}. Many of the  models (although not all) 
rely on autoregressive training and prediction, allowing for the sequential generation of sequence completions in a recursive manner conditioned on sequence history.

A natural question in this context is  how to compute answers to predictive queries that go beyond   traditional one-step-ahead predictions. Examples of such  queries are ``how likely is event $A$ to occur before event $B$?'' and ``how likely is event $C$ to occur (once or more) within the next $K$ steps of the sequence?'' These types of queries are very natural across a wide variety of application contexts, for example, the probability that an individual will 
finish speaking or writing a sentence within the next $K$ words, or that a user will use one app before another. See \cref{fig:flashy_example} and \cref{sec:app_modeling_details} for examples.

\begin{figure}
    \centering
    \includegraphics[width=0.8\textwidth]{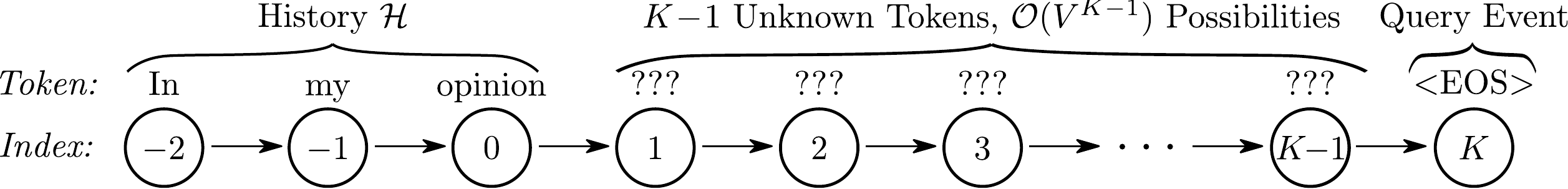} \\
    \vspace{0.5cm}
    \includegraphics[width=1.0\textwidth]{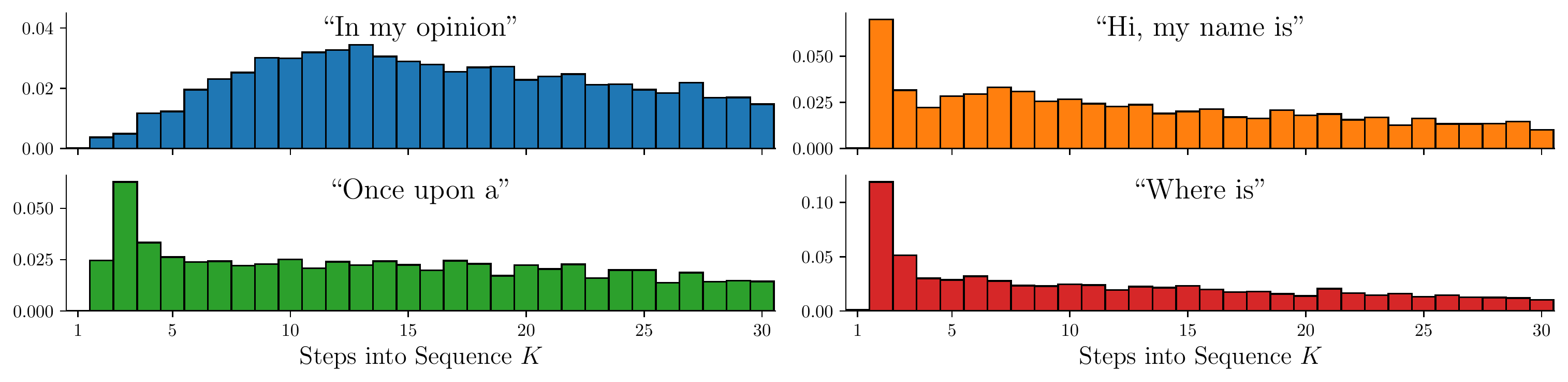}
    \caption{(top) Illustration of a query for the probability of a given sentence "In my opinion..." ending in $K$ steps. (bottom) GPT-2 query estimates across 4 prefixes with $V=50,257, K\leq30$. Importance sampling query estimates maintain a 6x reduction in variance relative to naive model sampling for the same computation budget. Open-ended prefixes (top-left) generally possess longer-tailed distributions relative to simple prefixes. Additional details provided in Sections 4, 5, and \cref{sec:gpt_exp}.}
    
    \label{fig:flashy_example}
\end{figure}


In this paper we develop a general framework for answering such predictive queries in the context of autoregressive (AR) neural sequence models. This amounts to computing conditional probabilities of propositional statements about future events, conditioned on the history of the sequence  as summarized by the current hidden state representation. We focus in particular on how to perform near real-time computation of such queries, motivated by use-cases such as answering human-generated queries and utilizing query estimates within the optimization loop of training a model. Somewhat surprisingly, although there has been extensive prior work on multivariate probabilistic querying in areas such as graphical models and database querying, as well as for restricted types of queries for traditional sequence models such as Markov models,  querying for neural sequence models appears to be  unexplored. One possible reason is that the problem is computationally intractable in the general case (as we discuss later in Section \ref{sec:queries}), typically scaling as $\bigO \! \left(V^{K-1}\right)$ or worse for predictions $K$-steps ahead, given a sequence model with a $V$-ary alphabet (e.g. compared to $\bigO \! \left(KV^2\right)$ for Markov chains).


Our contributions  are three-fold:
\begin{enumerate}[leftmargin=20pt]
\itemsep0em 
\item We introduce and develop the problem of predictive querying in neural sequence models by reducing complex queries to elementary building blocks. These elementary queries define restricted path spaces over future event sequences.  
\item We show that the underlying autoregressive model can always be constrained to the restricted path space satisfying a given query. This gives rise to a new proposal distribution that can be used for importance sampling, beam search, or a new hybrid approach. 
\item We evaluate these methods across three user behavior and two language data datasets. While all three methods significantly improve over naive forward simulation, the hybrid approach further improves over importance sampling and beam search. We furthermore explore how the performance of all methods relates to the model entropy.

\end{enumerate}  
Code for this work is available at \href{https://github.com/ajboyd2/prob_seq_queries}{https://github.com/ajboyd2/prob\_seq\_queries}.
\section{Related Work}
\label{sec:related}

Research on efficient computation of probabilistic queries has a long history in machine learning and AI, going back to   work on exact inference in multivariate  graphical models  \citep{pearl1988probabilistic,koller2009probabilistic}.
Queries in this context are typically of two types. The first are {\it conditional probability queries}, the focus of our attention here: computing probabilities defined for a subset $X$ of variables of interest, conditioned on a second subset $Y=y$ of observed variable values, and marginalizing over the set $Z$ of all other variables. The second type of queries can broadly be referred to as {\it assignment queries}, seeking the most likely (highest conditional probability) assignment of values $x$ for $X$, again conditioned on $Y=y$ and marginalizing over the set $Z$. Assignment queries are also referred to as most probable explanation (MPE) queries, or as maximum a posteriori (MAP) queries when $Z$ is the empty set \citep{koller2007graphical}.
 
For models that can be characterized with sparse  Markov dependence structure,   there is a significant body of work on efficient inference algorithms that can leverage such structure  \citep{koller2009probabilistic},  in particular for sequential models where   recursive computation can be effectively leveraged \citep{bilmes2010dynamic}. However, autoregressive neural sequence models are inherently non-Markov since the real-valued current hidden state is a function of the entire history of the sequence. Each hidden state vector induces a tree containing $V^k$ unique future trajectories with state-dependent probabilities for each path. Techniques such as dynamic programming (used effectively in Markov-structured sequence models) are not applicable in this context, and both assignment queries and conditional probability queries are NP-hard in general  \cite{chen2018recurrent}.

For assignment-type queries  there has been considerable work in natural language processing with neural sequence models, particularly for the MAP problem of generating high-quality/high-probability  sequences conditioned on sequence history or other conditioning information. A variety of heuristic decoding methods have been developed and found to be useful in practice, including   beam search \citep{sutskever2014sequence},  best-first search \citep{xu2021massive}, sampling methods \citep{holtzman2019curious}, and hybrid variants \citep{shaham2021you}.  However, for conditional probability queries with neural sequence models (the focus of this paper), there has been no prior work in general on this problem to our knowledge. While  decoding techniques such as beam search can also be useful in the context of conditional probability queries, as we will see later in Section \ref{sec:methods}, such techniques have significant limitations in this context, since by definition they produce lower-bounds on the probabilities of interest and, hence, are biased estimators.



\section{Probabilistic Queries}
\label{sec:queries}

\paragraph{Notation}
Let $X_{1:N} := [X_1, X_2, \dots, X_N]$ be a sequence of random variables with arbitrary 
length $N$. Additionally, let $x_{1:N} := [x_1, x_2, \dots, x_N]$ be their respective observed values where each $x_i$  takes on values from a fixed vocabulary $\mathbb{V} := \{1, \dots, V\}$. 
Examples of these sequences include sentences where each letter or word is a single value, or streams of discrete events generated by some process or user. 
We will refer to individual variable-value pairs in the sequence as events.

We consider an autoregressive model $p_\theta(X_{1:N}) = \prod_{i=1}^N p_\theta(X_i|X_{1:i-1})$, parameterized by $\theta$ and trained on a given dataset 
of $M$ independent draws from a ground truth distribution $\prob(X_{1:N})$. 
We assume that this model can be conditioned on a subsequence, termed the \emph{history} $\hist$. We will remap the indices of subsequent random variables to start at 1\footnote{For example, if $|\hist|=3$ then $\prob(X_4|\hist)$ is the distribution of the $7^{\text{th}}$ value in a sequence after conditioning on the first 3 values in the sequence.}.
We abbreviate conditioning on a history 
by an asterisk ${}^*$, i.e., $\prob^*(\cdot) := \prob(\cdot|\hist)$ and $p_\theta^*(\cdot) := p_\theta(\cdot|\hist)$.

\begin{figure}
    \centering
    \includegraphics[width=\textwidth]{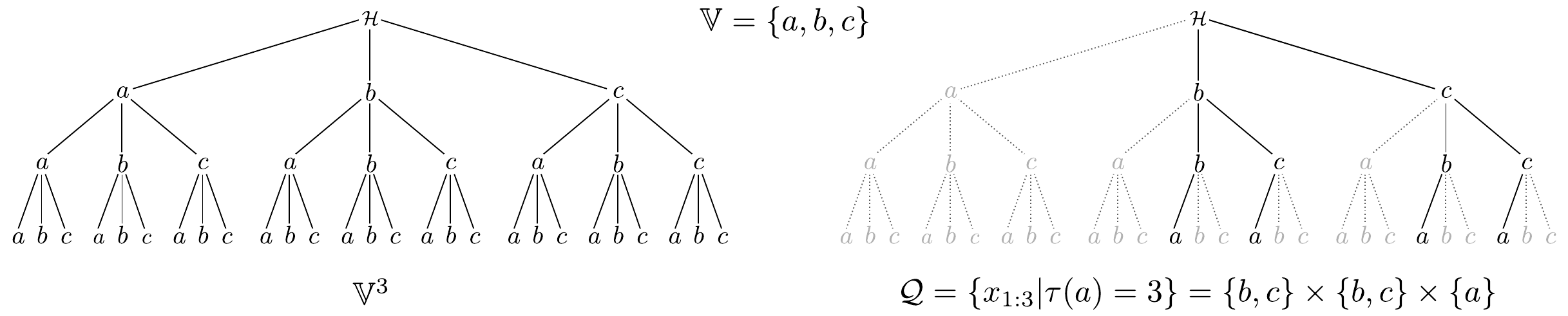}
    \caption{(left) Tree diagram of the complete sequence space for a vocabulary $\V=\{a,b,c\}$ and the corresponding query space $\mathcal{Q}$ (right) for when \emph{the first appearance of $a$} occurs on the third step (i.e., $\tau(a)=3$), defined as the set product of restricted domains listed below the figure.} 
    \label{fig:query_example}
\end{figure}

\paragraph{Defining Probabilistic Queries}
Given a specific history of events $\hist$, there are a variety of different questions one could ask about the future continuing where the history left off: ($\text{Q}1$) What event is likely to happen next? ($\text{Q}2$) Which events are   likely to occur $k>1$ steps from now? ($\text{Q}3$) What is the distribution of when the next instance of $a \in \V$ occurs? ($\text{Q}4$) How likely is it that we will see event $a\in\V$ occur before $b\in\V$? ($\text{Q}5$) 
How likely is it for $a\in\V$ to occur $n$ times in the next $k$ steps?

We define a common framework for such queries  by defining probabilistic queries to be of the form
$p_\theta^*(X_{1:K} \in \mathcal{Q})$
with $\mathcal{Q}\subset \V^K$. This 
can   be extended to the infinite setting (i.e., $p_\theta^*([X_k]_k \in \mathcal{Q})$ where $\mathcal{Q}\subset \V^\infty$). Exact computation of an arbitrary query is straightforward to represent:
\begin{align}
p_\theta^*(X_{1:K} \in \mathcal{Q}) & = \sum_{x_{1:K}\in\mathcal{Q}} p_\theta^*(X_{1:K} = x_{1:K})  = \sum_{x_{1:K}\in\mathcal{Q}} \prod_{k=1}^K p_\theta^*(X_{k} = x_{k} | X_{<k}=x_{<k}). \label{eq:exact_computations_cond}
\end{align}
Depending on $|\mathcal{Q}|$, performing this calculation can quickly become intractable, 
motivating lower bounds or approximations (developed in more detail in \cref{sec:methods}). In this context it is helpful to impose structure on the query $\mathcal{Q}$ to make subsequent estimation  easier, in particular by
breaking $\mathcal{Q}$ into the following structured partition:
\begin{align}
\mathcal{Q} & = \cup_i \mathcal{Q}^{(i)} \text{ where } \mathcal{Q}^{(i)} \cap \mathcal{Q}^{(j)} = \emptyset \text{ for } i\neq j \label{eq:structure_1} \\
\text{and } \mathcal{Q}^{(i)} & = \mathcal{V}^{(i)}_1 \times \mathcal{V}^{(i)}_2 \times \dots \times \mathcal{V}^{(i)}_K \text{ where } \mathcal{V}^{(i)}_k \subseteq \mathbb{V} \text{ for } k=1,\dots,K. \label{eq:structure_2}
\end{align}
In words, this means a given query $\mathcal{Q}$ can be broken into a partition of simpler queries $\mathcal{Q}^{(i)}$ which take the form of a set cross product between restricted domains $\mathcal{V}^{(i)}_k$, one domain for each token $X_k$.\footnote{Ideally, the partitioning is chosen to have the smallest number of $\mathcal{Q}^{(i)}$'s needed.} An illustration of an example query set can be seen in \cref{fig:query_example}. A natural consequence of this is that:
\begin{align*}
p_\theta^*(X_{1:K}\in\mathcal{Q}) & = \sum_i p_\theta^*\left(X_{1:K}\in\mathcal{Q}^{(i)}\right) = \sum_i p_\theta^*\left(\cap_{k=1}^K X_k \in\mathcal{V}^{(i)}_k \right),
\end{align*}
which lends itself to more easily estimating each term in the sum. This will be discussed in  \cref{sec:methods}.




\begin{table}[]
    \centering
    \begin{tabular}{clll}
    \toprule
    \# & Question & Probabilistic Query & Cost $\left(K\cdot|\mathcal{Q}|\right)$ \\
    \midrule
    \vspace{2pt} 
    \query{1} & Next event? & $p_\theta^*(X_1)$ & $\bigO (1)$\\
    \vspace{2pt} 
    \query{2} & Event $K$ steps from now? & $p_\theta^*(X_K)$ & $\bigO \! \left(V^{K-1}\right)$ \\
    \query{3} & Next instance of $a$? & $p_\theta^*(\tau(a)=K)$ & $\bigO \! \left((V-1)^{K-1}\right)$ \\
    \vspace{2pt} 
    \query{4} & Will $a$ happen before $b$? & $p_\theta^*(\tau(a) < \tau(b))$ &  $\bigO \! \left((V-2)^{K}\right)^\dagger$ \\
    \query{5} & How many instances of $a$ in $K$ steps? & $p_\theta^*(N_a(K)=n)$ &  $\bigO \! \left(\binom{K}{n}(V-1)^{K-n}\right)$ \\
    \bottomrule
    \end{tabular}
    \caption{List of example questions, corresponding probabilistic queries, and associated costs of exact computation computation with an autoregressive model. 
    The cost of accommodating a history $\hist$ is assumed to be an additive constant for all queries. 
    While \query{4} extends to infinite time, the cost reported is for computing a lower bound up to $K$ steps. }
    \label{table:prob_query_examples}
\end{table}

\paragraph{Queries of Interest}
All of the queries posed  earlier in this section can be represented under the framework detailed in \cref{eq:structure_1} and \cref{eq:structure_2}, as illustrated in \cref{table:prob_query_examples}.

$\textbf{Q}\mathbf{1}$ \& $\textbf{Q}\mathbf{2}$ \quad The queries $p^*_\theta(X_1=a)$ and $p^*_\theta(X_K=a)$ for some $a \in \V$ can be represented with $\mathcal{Q}=\{a\}$ and $\mathcal{Q}=\V^{K-1}\times \{a\}$ respectively. 

$\textbf{Q}\mathbf{3}$ \quad The probability of the next instance of $a \in \V$ occurring at some point in time $K\geq 1$, $p^*_\theta(\tau(a)=K)$ where $\tau(\cdot)$ is the \emph{hitting time}, can be represented as $\mathcal{Q}=(\V\setminus\{a\})^{K-1}\times\{a\}$. This can   be adapted for a set $A\subset\V$ by replacing $\{a\}$ in $\mathcal{Q}$ with $A$.

$\textbf{Q}\mathbf{4}$ \quad The probability of $a\in\V$ occurring before $b\in\V$, $p^*_\theta(\tau(a)<\tau(b))$, is represented as $\mathcal{Q}=\cup_{i=1}^\infty\mathcal{Q}^{(i)}$ where $\mathcal{Q}^{(i)}=(\V\setminus\{a,b\})^{i-1}\times\{a\}$. Lower bounds to this can be computed by limiting $i<i'$. Like Q3, this can also be extended to disjoint sets $A,B\subset\V$. 

$\textbf{Q}\mathbf{5}$ \quad The probability of $a\in\V$ occurring $n$ times in the next $K$ steps, $p^*_\theta(N_a(K)=n)$, is represented as $\mathcal{Q}=\cup_{i=1}^{C(K,n)} \mathcal{Q}^{(i)}$, where  $N_a(K)$ is a random variable for the number of occurrences of events of type $a$ from steps $1$ to $K$ and 
$\mathcal{Q}^{(i)}$'s are defined to cover all unique permutations of orders of products composed of: $\{a\}^n$ and $(\V\setminus\{a\})^{K-n}$.
Like above, this can easily be extended for $A\subset\V$.

\paragraph{Query Complexity}
From \cref{eq:exact_computations_cond}, exact computation of a query involves computing $K\cdot|Q|$ conditional distributions  (e.g., $p_\theta^*(X_k|X_{<k}=x_{<k})$) in an autoregressive manner. 
Under the structured representation, the number of conditional distributions needed is equivalently $\sum_i \prod_{k=1}^K |\mathcal{V}_k^{(i)}|$.
Non-attention based neural sequence models often define $p_\theta^*(X_k|X_{<k}=x_{<k}) := f_\theta(h_k)$ where $h_k = \text{RNN}_\theta(h_{k-1}, x_{k-1})$. As such, the computation complexity for any individual conditional distribution remains constant with respect to sequence length. We will refer to the complexity of this atomic action as being $\bigO(1)$. Naturally, the actual complexity depends on the model architecture and has a multiplicative scaling on the cost of computing a query. The number of atomic operations needed to exactly compute \query{1}-\query{5} for this class of models can be found in \cref{table:prob_query_examples}. Should $p_\theta$ be an attention based model (e.g., a transformer \citep{vaswani2017attention}) then the time complexity of computing a single one-step-ahead distribution becomes $\bigO(K)$, further exacerbating the \textbf{exponential growth} of many queries.
Note that with some particular parametric forms of $p_\theta$, query computation can be more efficient, e.g., see \cref{sec:markov_analysis} 
for a discussion on query complexity for Markov models.

\section{Query Estimation Methods}\label{sec:methods}

Since exact query computation can scale exponentially in $K$ 
it is natural to consider approximation methods. In particular we focus on importance sampling, beam search, and a hybrid of both methods. All methods will be based on a novel proposal distribution, discussed below. 

\subsection{Proposal Distribution}
\label{sec:proposal}
For various estimation methods which will be discussed later, it is beneficial to have a proposal distribution $q(X_{1:K}=x_{1:K})$ whose domain matches that of the query $\mathcal{Q}$. For importance sampling, we will need this distribution as a proposal distribution, while we use it as our base model for selecting high-probability sequences in beam search. 
We would like the proposal distribution to resemble our original model while also respecting the query. 
One thought is to have 
$q(X_{1:K}=x_{1:K}) = p^*_\theta(X_{1:K}=x_{1:K}|X_{1:K}\in\mathcal{Q})$. However, computing 
this probability involves normalizing over  $p^*_\theta(X_{1:K}\in\mathcal{Q})$ which is exactly what we are trying to estimate in the first place. 
Instead of restricting the \emph{joint} distribution to the query, we can instead restrict every \emph{conditional} distribution to the query's restricted domain. To see this, we 
first partition $\mathcal{Q}=\cup_i\mathcal{Q}^{(i)}$ and define an autoregressive proposal distribution for each $\mathcal{Q}^{(i)}=\prod_{k=1}^K \mathcal{V}^{(i)}_k$ as follows:
\begin{align}
q^{(i)}(X_{1:K}=x_{1:K})& =\prod_{k=1}^K p^*_\theta\left(X_k=x_k|X_{<k}=x_{<k},X_k\in\mathcal{V}_k^{(i)}\right) \label{eq:proposal_dist} \\    
 & = \prod_{k=1}^K \frac{p^*_\theta \! \left(X_k=x_k|X_{<k}=x_{<k}\right)\ind \!  \left(x_k\in\mathcal{V}_k^{(i)}\right)}{\sum_{v\in\mathcal{V}_k^{(i)}} \nonumber p^*_\theta(X_k=v|X_{<k}=x_{<k})}
\end{align}
where $\ind(\cdot)$ is the indicator function. That is, we constrain the outcomes of each conditional probability to the restricted domains $\mathcal{V}_k^{(i)}$ and renormalize them accordingly. To evaluate the proposal distribution's probability, we multiply all conditional probabilities according to the chain rule. 
Since the entire distribution is computed for a single model call $p^*_\theta(X_k|X_{<k}=x_{<k})$, it is possible to both sample a $K$-length sequence and compute its likelihood under $q^{(i)}$ with only $K$ model calls. Thus, we can efficiently sample sequences from a distribution that is both informed by the underlying model $p_\theta$ and that respects the given domain $\mathcal{Q}$. As discussed in the next section, this proposal will be used for importance sampling and for the base distribution on which beam search is conducted.

\subsection{Estimation Techniques}

\paragraph{Sampling}
One can naively sample any arbitrary probability value using Monte Carlo samples to estimate $p^*_\theta(X_{1:K}\in\mathcal{Q}) = \E_{p^*_\theta}[\ind(X_{1:K}\in\mathcal{Q})]$; however, this typically will have high variance. This can be substantially improved upon by exclusively drawing sequences from the query space $\mathcal{Q}$.
Arbitrary queries can be written as a sum of probabilities of individual sequences, as seen in \cref{eq:exact_computations_cond}. This summation can be equivalently written as an expected value,
\begin{align*}
p^*_\theta(X_{1:K}\in\mathcal{Q}) & = \sum_{x_{1:K}\in\mathcal{Q}}p^*_\theta(X_{1:K}=x_{1:K})  = |\mathcal{Q}| \; \E_{x_{1:K}\sim \mathcal{U}(\mathcal{Q})} \left[p^*_\theta(X_{1:K}=x_{1:K})\right],
\end{align*}
where $\mathcal{U}$ is a uniform distribution. It is common for $p_\theta^*$ to concentrate most of the available probability mass on a small subset of the total possible space $\V^K$. Should $|\mathcal{Q}|$ be large, then  $|\mathcal{Q}|\;p^*_\theta(X_{1:K}=x_{1:K})$ will have very large variance for $x_{1:K}\sim \mathcal{U}(\mathcal{Q})$. One way to mitigate this is to use importance sampling with our proposal distribution $q$ (Section~\ref{sec:proposal}):
\begin{align*}
p^*_\theta(X_{1:K}\in\mathcal{Q}) & = |\mathcal{Q}| \; \E_{x_{1:K}\sim \mathcal{U}(\mathcal{Q})} \left[p^*_\theta(X_{1:K}=x_{1:K})\right]  = \E_{x_{1:K}\sim q} \left[\frac{p^*_\theta(X_{1:K}=x_{1:K})}{q(X_{1:K}=x_{1:K})}\right] \\
& \approx \frac{1}{M} \sum_{m=1}^M \frac{p^*_\theta \! \left(X_{1:K}=x_{1:K}^{(m)}\right)}{q \! \left(X_{1:K}=x_{1:K}^{(m)}\right)} \quad\quad\text{ for } x_{1:K}^{(1)},\dots,x_{1:K}^{(M)}\overset{iid}{\sim} q.
\end{align*}
It is worth noting that this estimator could be further improved by augmenting the sampling process to produce samples without replacement from $q$ 
(e.g., \citep{meister2021conditional,pmlr-v97-kool19a,shi2020incremental}); for the purposes of the scope of this paper we focus on sampling with replacement.

\paragraph{Search}
An alternative to estimating a query by sampling is to instead produce a lower bound,
\begin{align*}
p^*_\theta(X_{1:K}\in\mathcal{Q}) = \sum_{x_{1:K}\in\mathcal{Q}}p^*_\theta(X_{1:K}=x_{1:K}) \geq \sum_{x_{1:K}\in\mathcal{B}}p^*_\theta(X_{1:K}=x_{1:K}),
\end{align*}
where $\mathcal{B} \subset \mathcal{Q}$.
In many situations, only a small subset of sequences $x_{1:K}$ in $\mathcal{Q}$ have a non-negligible probability of occurring due to the vastness of the total path space $V^K$. As such, it is possible for $|\mathcal{B}| \ll |\mathcal{Q}|$ while still having a minimal gap between the lower bound and the actual query value.

One way to produce a set $\mathcal{B}\subset\mathcal{Q}$ is through beam search.
To ensure that beam search only explores the query space, instead of searching with $p_\theta$ we utilize $q$ for ranking beams. Since beam search is a greedy algorithm and for a given conditional $q(X_k=a|X_{<k})\propto p_\theta^*(X_k=a|X_{<k})$ for $a\in\mathcal{V}_k$, the rankings will both respect the domain and be otherwise identical to using $p_\theta$ to rank.
Typically, the goal of beam search is to find the most likely completion of a sequence without having to explore the entire space of possible sequences. This is accomplished by greedily selecting the top-$B$ most likely next step continuations, or \emph{beams}, at each step into the future. Rather than finding a few high-likelihood beams, we are more interested in accumulating a significant amount of probability and less so in the specific quantity of beams collected.


Traditional beam search has a fixed beam size $B$ that is considered for its search; however, this is not ideal for collecting probability mass. 
%
As an alternative we develop \emph{coverage-based} beam search where at each step in a sequence we restrict the set of beams being considered not to the top-$B$ but rather to the smallest set of beams that collectively exceed a predetermined probability mass $\alpha$,  referred to as the ``coverage''.\footnote{This is similar to the distinction between top-$K$ and top-$p$ / nucleus sampling commonly used for natural language generation \citep{holtzman2019curious}.} More specifically, let $\mathcal{B}_k \subset \{x_{1:k}|x_{1:K}\in\mathcal{Q}\}$ be a set containing $|\mathcal{B}_k|$ beams for subsequences of length $k$. For brevity, we will assume that $\mathcal{Q}=\mathcal{V}_1\times \cdots \times \mathcal{V}_K$.\footnote{If $\mathcal{Q}$ requires partitioning into multiple $\mathcal{Q}^{(i)}$'s, we apply beam search to each sub query $p^*_\theta(X_{1:K}\in\mathcal{Q}^{(i)})$.} $\mathcal{B}_{k+1}$ is a subset of $\mathcal{B}_k \times \mathcal{V}_{k+1}$ and is selected specifically to minimize $|\mathcal{B}_{k+1}|$ such that $q(X_{1:k+1}\in\mathcal{B}_{k+1}) \geq \alpha$. It can be shown that $p_\theta^*(X_{1:K}\in\mathcal{Q})-p_\theta^*(X_{1:K}\in\mathcal{B}_K) \leq 1-q(X_{1:K}\in\mathcal{B}_K)$ (and is often quite less). See 
\cref{sec:bs_upper_proof}
for a proof.

There is one slight problem with having $\alpha$ be constant throughout the search. Since we are pruning based on the joint probability of the entire sequence, any further continuations of $\mathcal{B}_k$ will reduce the probability $q(X_{1:k+1}\in\mathbb{B}_{k+1})$. 
This may lead to situations in which every possible candidate sequence is kept in order to maintain minimal joint probability coverage.
This can be avoided by filtering by $\alpha_k$ where $\alpha_1 > \dots > \alpha_K=\alpha$, e.g., the geometric series $\alpha_k=\alpha^{k/K}$.

\paragraph{A Hybrid Approach}
Importance sampling produces an unbiased estimate, but can still experience large variance in spite of a good proposal distribution $q$ when $p_\theta$ is a heavy tailed distribution. Conversely, the beam search lower bound can be seen as a biased estimate with zero variance. We can remedy the limitations of both methods by recognizing that since $p^*_\theta(X_{1:K}\in\mathcal{Q}) = \sum_{x_{1:K}\in\mathcal{B}_K}p^*_\theta(X_{1:K}=x_{1:K}) +  \sum_{x_{1:K}\in\mathcal{Q}\setminus\mathcal{B}_K}p^*_\theta(X_{1:K}=x_{1:K})$ where $\mathcal{B}_K$ is the set of sequences resulting from beam search, we can use importance sampling on the latter summation. The only caveat to this is that the proposal distribution must match the same domain of the summation: $\mathcal{Q}\setminus\mathcal{B}_K$. 

The proposal distribution we use to accomplish this is $q(X_{1:K}=x_{1:K}|X_{1:K}\notin\mathcal{B}_K)$ (see Eq.~\ref{eq:proposal_dist}). This is implemented by storing all intermediate distributions (both original $p_\theta$ and proposal $q$) found during beam search and arranging them into a tree structure. All leaf nodes that are associated with $x_{1:K}\in\mathcal{B}_K$ have their transition probability zeroed out and then the effect of restriction and normalization is propagated up the tree to their ancestor nodes (details provided in \cref{sec:hybrid_details} and \ref{sec:app_hybrid_variance}). Sampling from this new proposal distribution involves sampling from this normalized tree until we reach either an existing $K$-depth leaf node, in which case we are done, or a $<K$-depth leaf node in which case we complete the rest of the sequence by sampling from $q$ in \cref{eq:proposal_dist} as usual.

Lastly, since our ultimate goal is to sample from the long tail of $p_\theta$, targeting a specific coverage $\alpha$ during beam search is no longer effective since achieving meaningfully large coverage bounds for non-trivial path spaces is generally intractable. Instead, we propose \emph{tail-splitting} beam search to better match our goals. Let $w_k^{(i)}=p^*_\theta(X_{1:k+1}=x^{(i)}_{1:k+1})$ for $x_{1:k+1}^{(i)}\in\mathcal{B}_k\times\mathcal{V}_{k+1}$ such that $w_k^{(i)} \geq w_k^{(j)}$ if $i < j$. In this regime, $\mathcal{B}_{k+1}=\{x^{(i)}_{1:k+1}\}_{i=1}^B$ where $B=\arg\min_{b} \sigma(w_k^{(1:b)})+\sigma(w_k^{(b+1:|\mathcal{B}_k\times\mathcal{V}_{k+1}|)})$. $\sigma(w_k^{(u:v)})$ is the empirical variance of $w_k^{(i)}$ for $i=u,\dots,v$. This can be seen as performing 2-means clustering on the $w_k$'s and taking the cluster with the higher cumulative probability.

\subsection{Saving Computation on Multiple Queries}
Should multiple queries need to be performed, such as $p_\theta^*(\tau(a)=k)$ for multiple values of $k$, then there is potential to be more efficient in computing estimates for all of them. The feasibility of reusing intermediate computations is dependent on the set of queries considered. 
For simplicity, we will consider two base queries $\mathcal{Q}=\mathcal{V}_1\times\dots\times\mathcal{V}_K$ and $\mathcal{Q}'=\mathcal{V}_1'\times\dots\times\mathcal{V}_{K'}'$ where $K < K'$. Due to the autoregressive nature of $p_\theta$, if $\mathcal{V}_i=\mathcal{V}_i'$ for $i=1,\dots,K-1$ then all of the distributions and sequences needed for estimating $p_\theta^*(X_{1:K}\in\mathcal{Q})$ are guaranteed to be intermediate results found when estimating $p_\theta^*(X_{1:K'}\in\mathcal{Q}')$. To be explicit, when estimating the latter query with beam search the intermediate $\mathcal{B}_{K}$ is the same as what would be directly computed for the former query. Likewise, for importance sampling if $x_{1:K'} \sim q(X_{1:K'})$ using \cref{eq:proposal_dist} then the subsequence $x_{1:K} \sim q(X_{1:K})$. This does not apply   when the sample path domain is further restricted, such as with the hybrid approach, in which case we cannot  directly use intermediate results to compute other queries for ``free.''

\section{Experiments and Results}
\label{sec:experiments}

\paragraph{Experimental Setting} 
We investigate estimates of hitting time queries across various datasets, comparing beam search, importance sampling, and the hybrid method. We find that hybrid systematically outperforms both pure search and sampling given a comparable computation budget across queries and datasets. We also investigate  the dependence of all three methods on the model entropy. 

It is worth noting that we focus almost exclusively on hitting time queries in our primary experiments as more complex queries \textit{often decompose into operations over individual hitting times}. For example. consider the following decomposition of the ``A before B'' query (Q4):
\begin{align*}
p_\theta^*(\tau(a) < \tau(b) ) = \sum_{k=1}^\infty p_\theta^* \!\left(\tau(a)=k, \tau(b)>k\right)
 = \sum_{k=1}^\infty p_\theta^* \! \left(X_k=a, X_{<k}\in (\V\setminus\{a,b\})^{k-1}\right) \label{eq:hit_comp_sum}
\end{align*}
Decompositions of other queries and their impact on approximation error are found in \cref{sec:query_decomposition}.

\paragraph{Datasets}
We evaluate our query estimation methods on three user behavior and two language sequence datasets. \textbf{Reviews} \citep{amazon-rev-ni-etal-2019-justifying} contains 
sequences of Amazon customer reviews for products belonging to one of $V=29$ categories; \textbf{Mobile Apps} \citep{app_dataset} consists of app usage records over $V=88$ unique applications; \textbf{MOOCs} consists of student interaction with online course resources over $V=98$ actions. We also use the works of William \textbf{Shakespeare} \citep{shakespeare_data} by modeling the occurrence of $V=67$ unique ASCII characters. Lastly, we examine \textbf{WikiText} \citep{wikitext} to explore word-level sequence modeling applications with GPT-2, a large-scale language model with a vocabulary of $V=50257$ word-pieces \citep{gpt2-radford, word-piece-Wu2016GooglesNM}. 
After preprocessing, no datasets contain personal identifiable information.

\paragraph{Base Models}
While our proposed methods are amenable to autoregressive models of arbitrary structure, we focus our analysis specifically on recurrent neural networks. For all datasets except WikiText, we train Long-short Term Memory (LSTM) networks until convergence. To explore modern applications of sequence modeling with WikiText, we utilize GPT-2 with pre-trained weights from HuggingFace \citep{hugging-face-wolf-etal-2020-transformers}. Model training and experimentation utilized roughly 200 NVIDIA GeForce 2080ti GPU hours. Please refer to \cref{sec:dataset_model_prep} for additional details. 

\paragraph{Experimental Methods} 
We investigate computation-accuracy trade-offs between 3 estimation methods (beam search, importance sampling, and the hybrid) across all datasets. Query histories $\mathcal{H}$ are defined by randomly sampling $N=1000$ sequences from the test split for all datasets except WikiText, from which we sample only $N=100$ sequences due to computational limitations. For each query history and method, we compute the hitting time query estimate $p_\theta^*(\tau(a)=K)$ over $K=3, \ldots , 11$, with $a$ determined by the $K^{th}$ symbol of the ground truth sequence. 

To ensure an even comparison of query estimators, we fix the computation budget per query in terms of model calls $f_\theta(h_k)$ to be equal across all 3 methods, repeating experiments for different budget magnitudes roughly corresponding to $O(10), O(10^2), O(10^3)$ model calls (see \cref{sec:additional_exps} for full details). We intentionally select relatively small computation budgets per query to support systematic  large-scale experiments over multiple queries up to relatively large values of $K$.  Results for queries with GPT-2   are further restricted because of computational limits and are reported separately below.

To evaluate the accuracy of the estimates for each query and method, we compute the true probability of $K$ using exact computation for small values of $K \leq 4$. For larger values of $K$, we run importance sampling with a large number of samples $S$, where $S$ is adapted per query to ensure the resulting 
unbiased estimate has an empirical variance less than $\epsilon\ll1$ (see \cref{sub:ground_truth_calc}). 
This computationally-expensive 
estimate is then used as a surrogate for ground truth in error calculations.

Coverage-based beam search is not included in our results: we found   that it experiences exponential growth with respect to $K$ and  does not scale efficiently due to its probability coverage guarantees. Additional details are provided in \cref{sub:coverage_bs_ablation}.

\begin{figure}
    \centering
    \includegraphics[width=1.0\textwidth]{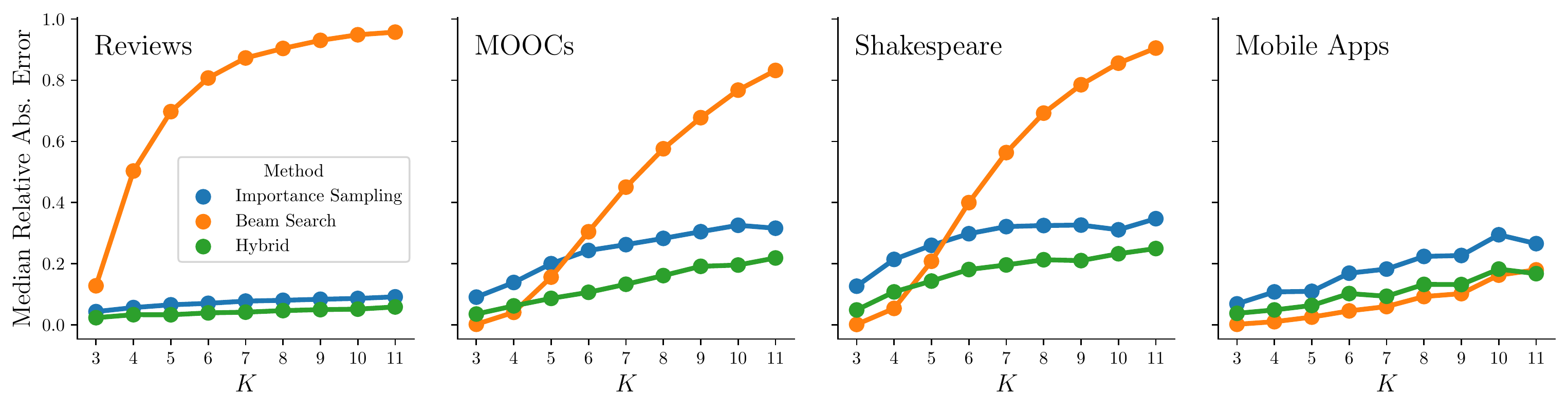}
    \caption{Median relative absolute error (RAE) between estimated probability and (surrogate) ground truth for $p^*_\theta(\tau(\cdot)=K)$ 
    for importance sampling, beam search, and the hybrid method. As query path space grows with $K$, beam search quickly fails to bound ground truth while sampling remains robust, with the hybrid consistently outperforming all other methods, especially for large values of $K$. Ground truth values used to determine error are exact for $K \leq 4$ and approximated otherwise.}
    \label{fig:err_plot}
\end{figure}

 
\paragraph{Results: Accuracy and Query Horizon}
Using the methodology described above, for each query, we compute the relative absolute error (RAE) $|p - \hat{p} |/p$, where $\hat{p}$ is the estimated query probability generated by a particular method and $p$ is the ground truth probability or the surrogate estimate using importance sampling. For each dataset, for each of the 3 levels of computation budget, for each value of $K$, this yields $N=1000$ errors for the $N$ queries for each method.  

\cref{fig:err_plot}  shows the median RAE of the $N$ queries, per method, as a function of $K$, for each dataset, using the medium  computation budget in terms of model calls.
Across the 4 datasets the error  increases systematically as $K$ increases. However,  beam search is significantly less robust than the other methods for 3 of the 4 datasets:  the error rate increases rapidly in a non-linear fashion compared to  importance sampling and hybrid. Beam search is also the most variable across datasets relative to other methods. The hybrid method systematically outperforms importance sampling along across all 4 datasets and  for all values of $K$.  In \cref{sec:additional_exps}  we provide additional results; for the lowest and highest levels of computational budget, for mean (instead of median) RAEs, and scatter plots for specific values of $K$ with more detailed error information. The qualitative conclusions are consistent across all plots.

\begin{figure}
    \centering
    \includegraphics[width=1.0\textwidth]{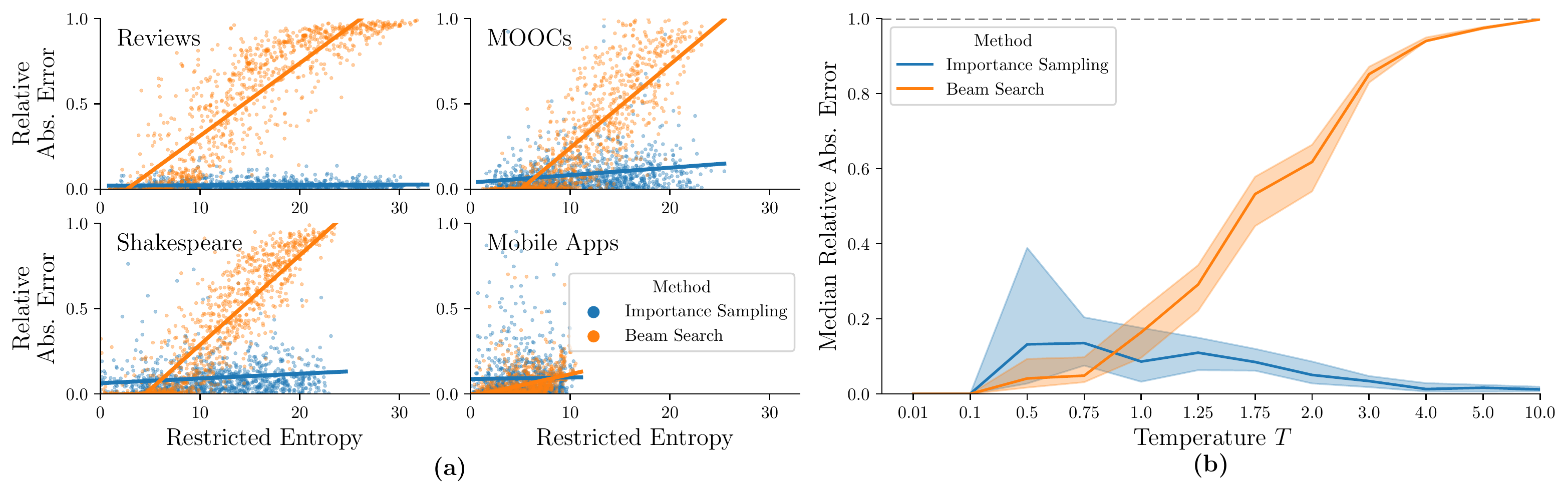}
    \caption{(left) RAE vs restricted entropy per query (with best linear fits), (right) Median RAE versus model temperature $T$ for Mobile App data. All errors computed using the same queries as in \cref{fig:err_plot}. Beam search errors correlate highly with model entropy even with the low-entropy Mobile Apps dataset, where increasing temperature $T$ directly induces this failure mode.
    } 
    \label{fig:temp_plots}
\end{figure}

\paragraph{Results: The Effect of Model Entropy}
We conjecture that the entropy of the proposal distribution $q$ conditioned on a given history $H_q^*(X_{1:K})=-\E_{x_{1:K}\sim q}[\log q(X_{1:K}=x_{1:K})]$, which we refer to as \emph{restricted entropy}, is a major factor in the performance of the estimation methods. 
\cref{fig:temp_plots}(a) shows the RAE per query (with a linear fit) as a function of estimated restricted entropy for importance sampling and beam search. The results clearly show that entropy is driving the query error in general and that the performance of beam search is much more sensitive to entropy than sampling. The difference in entropy characteristics across datasets explains the differences in errors we see in \cref{fig:err_plot}. In particular, the Mobile Apps dataset is in a much lower entropy regime than the other 3 datasets.

To further investigate the effect of entropy, we alter each model by applying a temperature $T>0$ to every conditional factor: $p_{\theta,T}(X_k|X_{<k}) \propto p_\theta(X_k|X_{<k})^{1/T}$, effectively changing the entropy ranges for the models. \cref{fig:temp_plots}(b) shows the median RAE, for query $p^*_{\theta,T}(\tau(\cdot)=4)$, as a function of model temperature for the Mobile Apps data. As predicted from   \cref{fig:temp_plots}(a), the increase in $T$, and corresponding increase in entropy, causes beam search's error to converge to 1, while the sampling error goes to 0. As $T$ increases, each individual sequence will approach having $1/|\mathcal{Q}|$ mass, thus needing many more beams to have adequate coverage. Results for other queries and the other 3 datasets (in \cref{sec:additional_exps}) further confirm the fundamental bifurcation of error between search and sampling (that we see in \cref{fig:temp_plots}(b))  as a function of entropy.

\paragraph{Results: Relative Efficiency of Proposal Distribution over Naive Query Estimation}

\begin{figure}
    \centering
    \includegraphics[width=\textwidth]{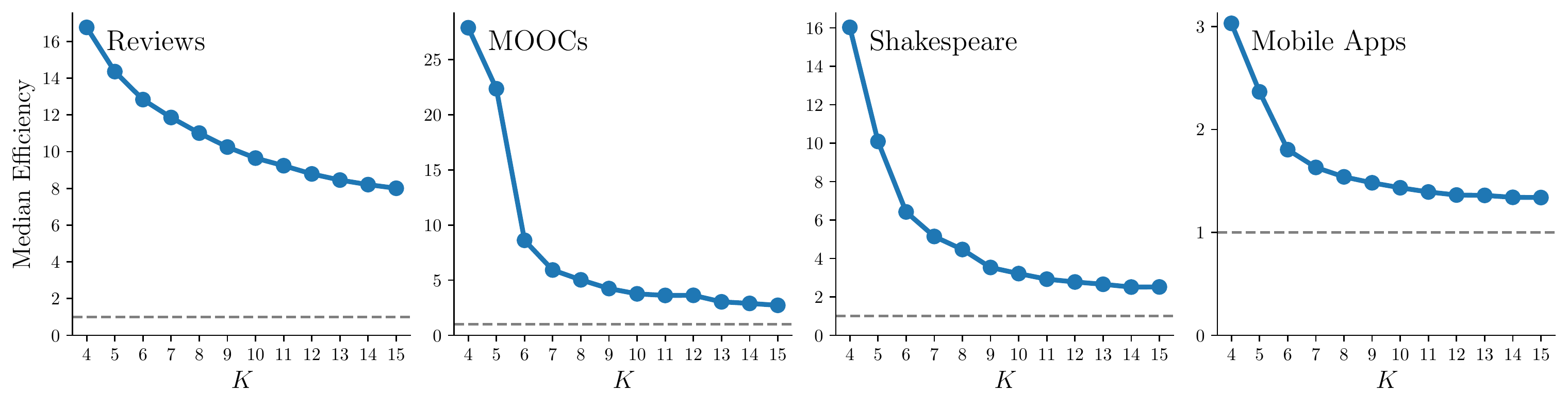}
    \caption{Median relative efficiency (over 1000 query histories and all vocabulary terms) of importance sampling estimation of the $K$-step marginal distributions for each dataset. The gray, dotted line represents 100\% relative efficiency defined by naive query estimation. Relative efficiency is documented for $ 4 \leq K \leq 15$ to highlight the regime where ground truth cannot be tractably computed.}
    \label{fig:relative_efficiency_abl}
\end{figure}

We also examine the relative efficiency improvements of our proposal distribution against naive Monte Carlo sampling: 
\begin{equation*}
p_\theta^{*}(X_{1:K} \in \mathcal{Q}) = \mathbb{E}_{x_{1:K} \sim p^*}[\mathbbm{1}(x_{1:K} \in \ \mathcal{Q})]
\end{equation*}

Our relative efficiency calculations in \cref{fig:relative_efficiency_abl} represent the variance ratio of naive query estimates and estimates from our query proposal distribution. As shown, all datasets witness improvement over naive sampling efficiency and often by a significant margin. We also observe that relative efficiency is largest for shorter query horizons, approaching naive sampling efficiency as $K$ increases.

\paragraph{Results: Queries with a Large Language Model}

We further explore entropy's effect on query estimation with GPT-2 and the WikiText dataset for $N=100$, $K = 3,4$, across 3 computation budgets. With a vocabulary 500x larger than the other models, GPT-2 allows us to examine queries relevant to NLP applications. 
The resulting high entropy causes beam search to fail  to match surrogate ground truth given the computation budgets (consistent with  earlier experiments), with a median RAE of 82\% (for $K=4$ and a budget of $O(10^3)$). By contrast, importance sampling's median RAE under the same setting is 13\%, \textbf{a 6x reduction}. 
Additional results  are in \cref{sec:gpt_exp}.

\section{Discussion} 
\label{sec:discussion}


\paragraph{Future Directions:}
This work provides a starting point for multiple different research directions beyond the scope of this paper. One such direction is exploring more powerful search and sampling methods for query estimation. This includes methods for sampling without replacement, such as the Gumbel-max method \citep{huijben2022review}, sampling importance-resampling (SIR) \citep{sampling_importance_resamp}, as well as new heuristics for automatically trading off search and sampling on a per-query basis. Another direction for exploration is amortized methods, such as learning models before queries are issued, that are specifically designed to help answer queries. Learning models that include queries as part of regularization terms in objective functions can also build on this work, e.g., learning  models that don't only rely on one-step-ahead losses \citep{meister2021language}. The work in this paper could also be broadened to continuous-time  models for marked temporal point processes (requiring marginalization over time as well as event types).


\paragraph{Limitations:}
Our results rely on only four datasets with a single autoregressive neural sequence model trained for each, naturally limiting the breadth of conclusions that we can draw. However, given that the results are consistent and validated by algorithm-independent entropy and efficiency analyses, we believe these findings have general validity and provide a useful starting point for others interested in the problem of querying neural sequence models. Another potential limitation is that our four datasets have relatively small vocabularies ($V = O(10^2)$, small by NLP standards at least); this choice was largely driven by computational limitations in terms of being able to conduct conclusive experiments (e.g. averaging over many event histories). Our (limited) results with GPT-2 provide clues on what may be achievable with much larger vocabularies: systematic analysis of querying for such models is a natural target for future work. Our work also does not address the issue of model inaccuracy: our results are entirely focused on computing  query estimates in terms of a model's distribution instead of the data-generating process. 
Exploring the effect of miscalibration errors in autoregressive models on  $k$-step-ahead query estimates is a promising avenue for future work.


\paragraph{Potential Negative Societal Impacts:}
Since the focus of this work is making predictions with data that is typically generated by individuals (language, online activity), there is naturally a potential for abuse of such methods. For example, if the underlying model in a system is miscalibrated, decisions could be made that negatively impact individuals. Examples include recommending a student be dropped from an online course if the model incorrectly predicts they will not participate in course modules. Even if the underlying model is well-calibrated,  predictive queries could potentially be used in a proactive manner to bias decisions against individuals whose event sequences are atypical, e.g., in a chatbot context,  inaccurately predicting future language patterns for certain individuals, leading to interruption and generation of an inappropriate response.

\paragraph{Acknowledgements:}
We thank the NeurIPS reviewers for their suggestions on
improving the original version of this paper. This work
was supported by National Science Foundation Graduate
Research Fellowship grant DGE-1839285 
(AB and SS), by an NSF CAREER Award (SM), by the National Science Foundation under award numbers 1900644 (PS), 2003237 (SM), and 2007719 (SM),
by the National Institute of Health under awards R01-AG065330-02S1 and R01-LM013344 (PS), by the Department of Energy under grant DE-SC0022331 (SM), by the HPI
Research Center in Machine Learning and Data Science at
UC Irvine (SS), and by Qualcomm Faculty awards (PS and SM).



\bibliography{citations}



\newpage

\appendix
\section{Query Estimation Process and Examples (Section 1 in paper)} \label{sec:app_modeling_details}


Below, we outline in clear terms the process of conducting query estimation experiments that we leverage in the main paper. Furthermore, we include demonstrative examples of potential queries of interest.

\paragraph{Query Estimation Process}
\begin{enumerate}
    \item Sample a sequence $x$ of length $n$ from the test set.
    \item Using a model $p_\theta$ that was trained on the training split of the dataset, condition on the first 5 elements, $x_{1:5}$.
    \item Then, using the proposed method of interest (e.g., importance sampling, coverage-based beam search, etc.), approximate the query of interest for the future continuation of the sequence–typically K steps into the future, $x_{6:6+K}$. As mentioned in lines 256 and 257 of the main paper, the main experiments pertain to hitting time queries $\tau(a)=K$ where $K=3,\dots,11$ and $a$ is determined as being equal to the actual value of $x_{6+K}$ from the sampled sequence. We ensure $6+K\leq n$.
    \item Repeat steps 1-3 over 1000 sequences (unless the dataset is WikiText, in which case do 100 sequences).
    \item Against either $p$ determined by absolute ground truth, or pseudo-ground truth $p$ computed via sampling, compute relative absolute error (RAE) of the form $\frac{|\hat{p} - p|}{p}$. For a given dataset, query estimation method, and budget, compute the median RAE over all 1000 (or 100 in the case of WikiTest) sampled sequences, where a sampled sequence is a history-target tuple $[\mathcal{H}, a]$. 
\end{enumerate}

\textbf{MOOCs} Query Example: 
\begin{itemize}[itemsep=0pt]
    \item \textbf{History} $\mathcal{H} = [ \texttt{log on}, \texttt{ open assignment}, \texttt{watch lecture} ]$
    \item \textbf{Query Class}: “A” before “B” query (Q4)
    \item \textbf{Formalism}: $p_\theta^*(\tau(A) < \tau(B)) = p_\theta^*(\tau(\texttt{purchase}) < \tau(\texttt{log off}))$
    \item \textbf{Description}: Given a user’s online class engagement history, what is the probability that they turn in their assignment before they navigate away from the page?
\end{itemize}

\textbf{Shakespeare} Query Example: 
\begin{itemize}[itemsep=0pt]
    \item \textbf{History} $\mathcal{H} = [ \text{“t”}, \text{“h”}, \text{“o”}, \text{“u”}, \text{“<space>”} ]$  
    \item \textbf{Query Class}: Combination query (Q5)
    \item \textbf{Formalism}: $p_\theta^*(N_{\texttt{vowel}}(10)=4)$
    \item \textbf{Description}: In the next 10 characters, what is the probability that 4 of them are vowels? 
\end{itemize}

\section{Query Decomposition (Section 3 in  paper)}
\label{sec:query_decomposition}
All of the questions posed in Table 1 in the main paper can be decomposed into readily available components that our model $p_\theta$ can estimate. Should we not utilize the proposed query structure ($\mathcal{Q}=\cup_i \mathcal{V}_1^{(i)}\times\dots\times\mathcal{V}_K^{(i)}$) then these are the derivations that would need to be performed individually in order to compute these values exactly on a case-by-case basis.

\paragraph{\bfquery{1}} $\prob^*(X_1)$ is already naturally in a form that our model can directly estimate due to the autoregressive factorization imposed by the architecture: $p_\theta^*(X_1)$. Furthermore, if we are interested in any potential continuing sequence $X_{1:K}=x_{1:K}$ this can easily be computed via $p_\theta^*(X_{1:K}=x_{1:K})=\prod_{k=1}^K p_\theta^*(X_k=x_k|X_{<k}=x_{<k})$.

\paragraph{\bfquery{2}} To evaluate $p_\theta^*(X_K)$, terms $X_{<K}$ need to be marginalized out. This is naturally represented like so:
\begin{align*}
p_\theta^*(X_K) & = \E_{x_{<K} \sim p_\theta^*(X_{<K})}\left[p_\theta^*(X_K|X_{<K}=x_{<K})\right] \\
& = \sum_{x_{<K}\in \V^{K-1}} p_\theta^*(X_K, X_{<K}=x_{<K})
\end{align*}
It is helpful to show  both the exact summation form as well as the expected value representation as both will be useful in \cref{sec:methods}.

\paragraph{\bfquery{3}} The ``hitting time'' or the next occurrence of a specific event type $a\in\V$ is defined as $\tau(a)$. Evaluating the distribution of the hitting time with our model can be done like so:
\begin{align*}
p_\theta^*(\tau(a)=K) & = p_\theta^*(X_K=a, X_{<K}\neq a), \\
& = \sum_{x_{<K}\in \{\V\setminus\{a\}\}^{K-1}}p_\theta^*(X_K=a, X_{<K}=x_{<K}).
\end{align*}
The value $a\in\V$ can be easily replaced with a set of values $A\subset \V$ in these representations.
\begin{align*}
p_\theta^*(\tau(A)=k) & = p_\theta^*(X_K\in A, X_{<K}\in (\V\setminus A)^{K-1}), \\
& = \sum_{x_{<K}\in \{\V\setminus A\}^{K-1}} \sum_{a\in A} p_\theta^*(X_K=a, X_{<K}=x_{<K})
\end{align*}
Interestingly, we can see that \query{3} is a generalization of \query{2} by noting that they are identical when $A=\{\}$. 

\paragraph{\bfquery{4}} To evaluate $p_\theta^*(\tau(a) < \tau(b) )$, we must consider the possible instances where this condition is fulfilled. In doing so, we end up evaluating multiple queries similar in form to \query{3}, like so:
\begin{align}
p_\theta^*(\tau(a) < \tau(b) ) & = \sum_{k=1}^\infty p_\theta^*(\tau(a)=k, \tau(b)>k) \nonumber \\
& = \sum_{k=1}^\infty p_\theta^*(X_k=a, X_{<k}\in (\V\setminus\{a,b\})^{k-1}) \label{eq:hit_comp_sum}
\end{align}
In practice, computing this exactly is intractable due to it being an infinite sum. There are two potential approaches one could take to subvert this. The first of which is to ask a slightly different query:
\begin{align*}
    p_\theta^*(\tau(a) < \tau(b)| \tau(a) \leq K) & = \frac{p_\theta^*(\tau(a) < \tau(b), \tau(a) \leq K)}{p_\theta^*(\tau(a) \leq K, \hist_T)} \\
    & = \frac{\sum_{k=1}^{K} p_\theta^*(X_k=a, X_{<k}\in (\V\setminus \{a,b\})^{k-1})}{\sum_{k=1}^{K} p_\theta^*(\tau(a)=k)}.
\end{align*} 
The other option is to produce a lower bound on this expression by evaluating the sum in \cref{eq:hit_comp_sum} for the first $K$ terms. We can achieve error bounds on this estimate by noting that $p_\theta^*(\tau(a) < \tau(b) ) + p_\theta^*(\tau(a) > \tau(b) ) = 1$. As such, if we evaluate \cref{eq:hit_comp_sum} up to $K$ terms for both $p_\theta^*(\tau(a) < \tau(b) )$ and $p_\theta^*(\tau(b) < \tau(a) )$, the difference between the sums will be the maximum error either lower bound can have.

Similar to Q3, we can also ask this query with sets $A\perp B\subset \V$ instead of values $a,b$. 
\begin{align*}
p_\theta^*(\tau(A) < \tau(B) ) & = \sum_{k=1}^\infty p_\theta^*(\tau(A)=k, \tau(B)>k) \\
& = \sum_{k=1}^\infty \sum_{a\in A}p_\theta^*(X_k=a, X_{<k}\in (\V\setminus(A\cup B))^{k-1})
\end{align*}

\paragraph{\bfquery{5}} Evaluating $p^*_\theta(N_a(K)=n)$ also involves decomposing this value of interest into statements involving hitting times. 

Let $\tau^{(l)}(a)$ be the $l^\text{th}$ hitting time of a specific event of interest $a$. In other words, the time of the $l^\text{th}$ occurrence of $a$. Assuming $n \leq K$:
\begin{align*}
p^*_\theta&(N_a(K)=n)\\
& = \sum_{i_1<i_2<\dots<i_n\leq K} p^*_\theta(\tau^{(1)}(a)=i_1, \tau^{(2)}(a)=i_2,\dots, \tau^{(n)}(a)=i_n, \tau^{(n+1)}(a)>K) \\
& = \sum_{i_1<i_2<\dots<i_n\leq K} p^*_\theta(X_{i_1}=\tau^{(1)}(a)=i_1, \tau^{(2)}(a)=i_2,\dots, \tau^{(n)}(a)=i_n, \tau^{(n+1)}(a)>K)
\end{align*}

\paragraph{Approximation Errors} With queries that fall under the Q5 category (e.g., how likely will the event $a$ happen $n$ times in the next $K$ steps? $p_\theta^*( N_{a}(K)=n)$), while in the limiting case C(K,n) will become the dominant scaling factor for the size of the query space, in many circumstances this actually isn’t as impactful as one might expect. For instance, it can be shown that the query space of a Q5 query is larger than that of a hitting time query (Q3) evaluated one step further (at $K+1$) if $\sqrt[n]{C(K,n)} > V$ where $V$ is the vocabulary size. Looking at our main set of experiments from \cref{sec:experiments}, we can see that when evaluating hitting time queries up to 11 steps into the future ($K=11$), these queries had a slightly larger query space of than that of $p_\theta^*(N_{a}(K)=n)$ for $K=10, \forall_{0\leq n \leq K}$ and thus can be seen as comparable in complexity. This behavior easily extends to larger values of $K$ (e.g., $K \approx 100$), especially for larger vocabulary sizes. This is due to the $n^\text{th}$ root dramatically slowing down the factorial growth in the inequality shared above.

\section{Complexity of Predictive Querying with First-Order Markov Models (Section 3 in  paper)}
\label{sec:markov_analysis}

As a point of reference for neural sequence models we summarize below  the cost of various queries for Markov models (using basic results from the theory of finite Markov chains, e.g., see \citep{kemeny1983finite}).
Consider a first-order ergodic homogeneous Markov chain with a $V \times V $ transition matrix with elements $p(x_{k+1} | x_k )$. We analyze below the complexity of various queries for such a chain, conditioned on the current observation $X_0 = v$, where $v \in \V$ (this  is analogous to conditioning on $\hist$ for neural sequence models). We use $p( \cdot | v)$ as shorthand for $p( \cdot | X_0 = v)$.

\paragraph{Q2: computing $p(X_k |  v)$}
This can be computed recursively by computing the conditional distrubution $p(X_1|v)$, then $p(X_2|v) = \sum_{v'} p(X_2|v') p(v'|v)$, and so on, resulting in $k$ matrix multiplications, with complexity $O( k V^2)$.

\paragraph{Q3: computing $p(\tau(a) = k|  v)$}
For finite $k$, we can  compute the result for all values $k' \le k$ using $k-1$ matrix multiplications where,  at each step $k' = 2, \ldots, k-1$, all states (events) except $a$ are marginalized over. Thus, the complexity is  $O(k V^2)$.

In general it is straightforward to show that 
\begin{align*}   
p(\tau(a) = k|  v) = p(a|v) , \ \ \ k=1
\end{align*}
and
\begin{align*} 
p(\tau(a) = k|  v) = p(a| \bar{a}) p(\bar{a} | v) p(\bar{a} | \bar{a} )^{(k-2)}, \ \ \ k>1
\end{align*}
where $p(a | \bar{a} )$ is the probability of transitioning to state $a$ given that the chain is not in state $a$. Computation of $p(a | \bar{a} )$ requires knowledge of the steady-state probabilities of the chain $\pi_a$ and $\pi_{\bar{a}} = \sum_{v' \ne a} \pi_{v'}$. Computing the steady-state probabilities requires inversion of a $V \times \V$ matrix, with a complexity of $O(V^3)$. This general solution will be faster to compute than the version using matrix multiplication (up to horizon $k$) whenever $k > V$.

\paragraph{Q4: computing $p(\tau(a) < \tau(b)|  v)$}
Let $c$ be the set of all states (events) in $\V$ except for $a$ and $b$. It is straightforward to show that
\begin{align*}
p(\tau(a) < \tau(b) | v) = p(a|v) + p(c|v) \frac{ p(a|c) }{p(a|c) + p(b|c)}
\end{align*}
and
\begin{align*}
p(\tau(b) < \tau(a) | v) = p(b|v) + p(c|v) \frac{ p(b|c) }{p(a|c) + p(b|c)}
\end{align*}
where $p(a|c)$ and $p(b|c)$ are the probabilities of transitioning to $a$ and $b$, respectively, given that the system is currently in a state that is neither $a$ or $b$. Computing these probabilities again requires knowledge of the steady-state probabilities $\pi_a, \pi_b, \pi_c$, resulting in $O(V^3)$ time complexity.


\paragraph{General Queries and Higher Order Markov Models}
For simplicity, we will assume that a query takes the form $\mathcal{Q}=\mathcal{V}_1\times\dots\times\mathcal{V}_K$ and we are interested in $p(X_{1:K}\in\mathcal{Q}|\hist)$ for a $m^\text{th}$-order Markov model $p$. This model can be defined with an $(m+1)$-dimensional tensor $\Pi\in\mathbb{R}^{V\times\dots\times V}$ with elements $\pi_{i_1,\dots,i_m,i_{m+1}}$ such that $\sum_{j=1}^V \pi_{i_1,\dots,i_m,j} = 1$ for all $i_1,\dots,i_m\in \V$. Alternatively, $\pi_{i_1,\dots,i_m,i_{m+1}}=p(X_{j+m+1}=i_{m+1}|X_{j+1:j+m}=i_{1:m})$ for $j\geq 0$.

To marginalize out $X_{m+1}$ and compute the conditional distribution $p(X_{m+2}|X_{1:m})$, that requires the following computation:
\begin{align*}
p(&X_{m+2}=x_{m+2}|X_{1:m}=x_{1:m}) = \sum_{v\in\V} p(X_{m+2}=a,X_{m+1}=v|X_{1:m}=x_{1:m}) \\
& = \sum_{v\in\V} p(X_{m+2}=x_{m+2}|X_{m+1}=v,X_{2:m}=x_{2:m})p(X_{m+1}=v|X_{1:m}=x_{1:m}) \\
& = \sum_{v\in\V} \pi_{x_2,\dots,x_m,v,a}\pi_{x_1,\dots,x_{m-1},x_m,v} 
\end{align*}
If we perform this over all values of $x_{1:m}$, we can construct a new transition tensor representing $p(X_{m+2}|X_{1:m})$. We will denote this new tensor as being equal to $\Pi \otimes \Pi$ where $\left(\Pi \otimes \Pi\right)_{i_1,\dots,i_m,v}=p(X_{j+m+2}=v|X_{j+1:j+m}=i_{1:m})$. This operation has a computation complexity of $\mathcal{O}(V^{m+1})$.

This special product can be done repeatedly to further marginalize out. For instance, performing this operation on $\Pi$ $(n-1)$-times results in $(\Pi\otimes\dots\otimes\Pi)_{i_1,\dots,i_m,v}=p(X_{j+m+n}=v|X_{j+1:j+m}=i_{1:m})$, thus marginalizing out $n-1$ terms: $X_{j+m+1:j+m+n-1}$.

Transitioning into a restricted space $\mathcal{V}$ can be done easily by defining a restricted transition tensor $\Pi_{\mathcal{V}}$ such that $\left(\Pi_{\mathcal{V}}\right)_{i_1,\dots,i_m,v} \propto \pi_{i_1,\dots,i_m,v}$ if $v \in \V$, otherwise $\left(\Pi_{\mathcal{V}}\right)_{i_1,\dots,i_m,v}=0$ for all $i_1,\dots,i_m,v \in \V$. 

With this, we have everything we need to compute $p(X_{1:K}\in\mathcal{Q}|\hist)$. If the last $m$-values of the history $\hist$ are equal to $i_1,\dots,i_m$, then:
\begin{align*}
p(X_{1:K}\in\mathcal{Q}|\hist) &= \prod_{k=1}^K p(X_k\in\mathcal{V}_k|X_{<k}\in\mathcal{V}_1\times\dots\times\mathcal{V}_{k-1})\\
&= \prod_{k=1}^K\sum_{v\in\mathcal{V}_k} p(X_k=v|X_{<k}\in\mathcal{V}_1\times\dots\times\mathcal{V}_{k-1})\\
& = \prod_{k=1}^K \sum_{v\in\mathcal{V}_k} \left(\Pi_{\mathcal{V}_1}\otimes\dots\otimes\Pi_{\mathcal{V}_k}\right)_{i_1,\dots,i_m,v}.
\end{align*}
The dominant factor in the computational complexity of this is computing all of the special products of $\Pi$. As such, this has a total computational complexity of $\mathcal{O}((k-1)V^{m+1})$. 





\section{Details for the Hybrid Method (Section 4 in paper)}
\label{sec:hybrid_details}

The hybrid method estimates an arbitrary probabilistic query $p^*_\theta(X_{1:K}\in\mathcal{Q})$ by first using some variant of beam search (in our results we use our proposed \emph{tail-splitting} beam search, see Section 3) to find a lower bound $p^*_\theta(X_{1:K}\in\mathcal{B})$, and then using a sampling method (we use importance sampling, see Section 3) to estimate the remainder $p^*_\theta(X_{1:K}\in\mathcal{Q}\setminus\mathcal{B})$. The only complicated portion of this process lies in how we derive an acceptable proposal distribution for the sampling phase that respects the domain $\mathcal{Q}\setminus\mathcal{B}$. The following paragraphs detail how we construct such a distribution.

\paragraph{Viewing $\V^K$ and $p^*_\theta$ as  Trees}
In the space of all possible sequences of length $K$, $X_{1:K}\in\V^K$, one can represent these sequences as paths in a tree. Each node in this tree represents a single element in a sequence $X_k\in\V$ with depth $k$, with parent and children nodes representing previous and potential future values in the sequence respectively. The root node either represents the very beginning of a sequence, or a concrete history $\hist$ to condition on.

This tree can be augmented into a probabilistic one by defining edges between nodes as the conditional probability of a child node being next in a sequence, conditioned on all ancestors of that child. These probabilities are naturally determined by $p^*_\theta(X_k=x_k|X_{<k}=x_{<k})$ where $x_k$ is the child node value and $x_{<k}$ are the ancestors' values. 

\paragraph{Building the Tree}
Any subset of $\V^K$ can also be represented as a tree, and in fact will be a sub-tree of the one that represents $\V^K$. As such, there exists a tree that represents $\mathcal{Q}$. Our usual proposal distribution $q$ is a natural source of conditional probabilities for the edges. While none of these trees with their edge weights are fully known ahead of time, we do explore and uncover them through the process of beam search. As such, during the beam search phase of the hybrid method we keep track of any conditional distributions $q(X_k=x_k|X_{<k}=x_{<k})$ that are computed and use them to construct a partial view of the tree for $\mathcal{Q}$. Note that the end result of this process is a tree that will likely have many paths that do not fully reach depth $K$; however, there will be at least $|\mathcal{B}|$ many that do.

For our purposes, it is also useful to keep track of $p^*_\theta(X_k=x_k|X{<k}=x_{<k})$ over this restricted set, as well as model byproducts such as hidden states to reduce computation redundancy later. 

\paragraph{Pruning the Tree}
After beam search has completed, we are left with a resulting set of beams $\mathcal{B}$ and a probabilistic tree representing $q$. We would now like to alter this tree such that its weights represent $q_\mathcal{B}(X_{1:K}=x_{1:K}):=q(X_{1:K}=x_{1:K}|X_{1:K}\notin\mathcal{B})$. This alteration can be accomplished by adjusting the edge weights in the tree recursively as detailed below. New weight assignments will be denoted by $q'$ to differentiate from old weights $q$. The steps to the procedure are defined as follows:
\begin{enumerate}[leftmargin=20pt]
\itemsep0em 
\item At the final depth $K$, assign edge weights $q'(X_K=x_K|X_{<K}=x_{<K})=0$ for all $x_{1:K}\in\mathcal{B}$. All other edge weights in the final depth will have new weights $q'(X_K=x_K|X_{<K}=x_{<K})=q(X_K=x_K|X_{<K}=x_{<K})$. 
\item At the next layer above with depth $k=K-1$, assign edge weights as $q'(X_k=x_k|X_{<k}=x_{<k})=q(X_k=x_k|X_{<k}=x_{<k})\sum_{v\in\V}q'(X_{k+1}=v|X_{\leq k}=x_{\leq k})$ for all sub-sequences $x_{1:k}\in\mathcal{B}$.  
\item Repeat step 2 iteratively for $k=K-2, K-3, \dots, 2, 1$. 
\item Finally, normalize every conditional distribution for every node whose children edges were altered such that they each sum to 1.
\end{enumerate}
After these steps are completed, $q_\mathcal{B}(X_{1:K}=x_{1:K})=\prod_{k=1}^K q'(X_k=x_k|X_{<k}=x_{<k})$. Note that weights related to sequences that were not discovered during beam search, and are thus not in the tree, are not altered and still match the original proposal distribution $q$. As such, to sample sequences from the tree, we start at the root node and sample from each successive conditional distribution until either depth $K$ or a leaf node at depth $k<K$ is reached. In the former scenario, the sampling is complete. In the latter, the remaining values of the sequence are sampled from $q(\cdot|X_{\leq k}=x_{\leq k})$ like usual.

\section{Variance of Estimates from the Hybrid Method (Section 4 in paper)}
\label{sec:app_hybrid_variance}

We are assuming to be under the hybrid method regime where a collection of sequences $\mathcal{B}\subset\mathcal{Q}$ relevant to answering $p^*_\theta(X_{1:K}\in\mathcal{Q})$ have been deterministically found and are interested in using sampling methods to estimate the remainder $p^*_\theta(X_{1:K}\in\mathcal{Q}\setminus\mathcal{B})$. For brevity, we will assume that $\mathcal{Q}=\mathcal{V}_1\times\dots\times\mathcal{V}_K$. As mentioned in the previous section, we leverage our originally presented proposal distribution $q$ by further restricting the domain to $\mathcal{Q}\setminus\mathcal{B}$ in order to be used in this scenario. This will be represented by
\begin{align*}
q_\mathcal{B}(X_{1:K}=x_{1:K}):\!&=q(X_{1:K}=x_{1:K}|X_{1:K}\notin\mathcal{B}) \\
&=\frac{q(X_{1:K}=x_{1:K})\ind(x_{1:K}\in\mathcal{Q}\setminus\mathcal{B})}{1-q(X_{1:K}\in\mathcal{B})}.
\end{align*}
Note that an associated autoregressive form $q_\mathcal{B}(X_k=x_k|X_{<k}=x_{<k})$ exists and is well defined; however, the exact definition is a bit unwieldy. Please refer to the previous section for more details.

Using this proposal distribution, we can easily estimate the remaining probability:
\begin{align*}
    p^*_\theta(X_{1:K}\in\mathcal{Q}\setminus\mathcal{B}) & = \E_{x_{1:K}\sim q_\mathcal{B}}\left[\frac{p^*_\theta(X_{1:K}=x_{1:K})}{q_\mathcal{B}(X_{1:K}=x_{1:K})}\right] \\
& \approx \frac{1}{M} \sum_{m=1}^M \frac{p^*_\theta(X_{1:K}=x_{1:K}^{(m)})}{q_\mathcal{B}(X_{1:K}=x_{1:K}^{(m)})} \quad\quad\text{ for } x_{1:K}^{(1)},\dots,x_{1:K}^{(M)}\overset{iid}{\sim} q_\mathcal{B}.
\end{align*}
Let $\omega_\mathcal{B}(x_{1:K}):=\frac{p^*_\theta(X_{1:K}=x_{1:K})}{q_\mathcal{B}(X_{1:K}=x_{1:K})}$ and for brevity we will refer to $p^*_\theta(X_{1:K}\in\mathcal{Q}\setminus\mathcal{B})$ as  $\bar{\omega}_\mathcal{B}$.

If we assume $\mathcal{B}'=\mathcal{B}\cup\{\hat{x}_{1:K}\}$ for some $\hat{x}_{1:K}\in\mathcal{Q}\setminus\mathcal{B}$, then it is interesting to determine when exactly there will be a reduction in sampling variance for $p^*_\theta(X_{1:K}\in\mathcal{Q}\setminus\mathcal{B}')$ versus $p^*_\theta(X_{1:K}\in\mathcal{Q}\setminus\mathcal{B})$ as this will give insight into when the hybrid method is successful. 
In other words, we would like to show when the following inequality holds true:
\begin{align*}
\Delta_{\text{Var}}:=\var_{x_{1:K}\sim q_{\mathcal{B}}} \left[\omega_{\mathcal{B}}(x_{1:K})\right] - \var_{x_{1:K}\sim q_{\mathcal{B}'}} \left[\omega_{\mathcal{B}'}(x_{1:K})\right] \geq 0
\end{align*}
If this is true for a given $\hat{x}_{1:K}$, then this finding can be applied recursively for more general (but still possibly restricted) $\mathcal{B}'\supset\mathcal{B}$.
\begin{align*}
\var_{x_{1:K}\sim q_{\mathcal{B}}} \left[\omega_{\mathcal{B}}(x_{1:K})\right] & = \E_{x_{1:K}\sim q_\mathcal{B}} \left[\left(\omega_\mathcal{B}(x_{1:K}) - \bar{\omega}_\mathcal{B}\right)^2\right] \\
& = \E_{x_{1:K}\sim q_\mathcal{B}} \left[\omega_\mathcal{B}(x_{1:K})^2\right] - \bar{\omega}_\mathcal{B}^2 \\
& = \sum_{x_{1:K}\in\mathcal{Q}\setminus\mathcal{B}} q_{\mathcal{B}}(x_{1:K})\omega_\mathcal{B}(x_{1:K})^2 - \bar{\omega}_\mathcal{B}^2
\end{align*}
It then follows that
\begin{align}
\Delta_\text{Var} & =
\bar{\omega}_{\mathcal{B}'}^2 - \bar{\omega}_\mathcal{B}^2 +
\sum_{x_{1:K}\in\mathcal{Q}\setminus\mathcal{B}}  q_{\mathcal{B}}(x_{1:K})\omega_\mathcal{B}(x_{1:K})^2 - \sum_{x_{1:K}\in\mathcal{Q}\setminus\mathcal{B}'} q_{\mathcal{B}'}(x_{1:K})\omega_{\mathcal{B}'}(x_{1:K})^2 \nonumber \\
& = (\bar{\omega}_{\mathcal{B}'}^2 - \bar{\omega}_\mathcal{B}^2) +
q_\mathcal{B}(\hat{x}_{1:K})\omega_{\mathcal{B}}(\hat{x}_{1:K})^2 + \nonumber \\
& \quad\quad \sum_{x_{1:K}\in\mathcal{Q}\setminus\mathcal{B}'}\left[q_{\mathcal{B}}(x_{1:K})\omega_\mathcal{B}(x_{1:K})^2- q_{\mathcal{B}'}(x_{1:K})\omega_{\mathcal{B}'}(x_{1:K})^2\right] \label{eq:three_terms}
\end{align}
We will now analyze each of the three terms in \cref{eq:three_terms} to determine when $\Delta_\text{Var}\geq0$. For the first term, it follows that:
\begin{align}
\bar{\omega}_{\mathcal{B}'} + \bar{\omega}_\mathcal{B} & = p^*_\theta(X_{1:K}\in\mathcal{Q}\setminus\mathcal{B}') +  p^*_\theta(X_{1:K}\in\mathcal{Q}\setminus\mathcal{B})\nonumber  \\
& = 2p^*_\theta(X_{1:K}\in\mathcal{Q}\setminus\mathcal{B}') + p^*_\theta(\hat{x}_{1:K}) \nonumber \\
\bar{\omega}_{\mathcal{B}'} - \bar{\omega}_\mathcal{B} & = p^*_\theta(X_{1:K}\in\mathcal{Q}\setminus\mathcal{B}') - p^*_\theta(X_{1:K}\in\mathcal{Q}\setminus\mathcal{B}) \nonumber \\
& = -p^*_\theta(\hat{x}_{1:K}) \nonumber \\
\implies \bar{\omega}_{\mathcal{B}'}^2 - \bar{\omega}_\mathcal{B}^2 & = -p^*_\theta(\hat{x}_{1:K})\left(2p^*_\theta(X_{1:K}\in\mathcal{Q}\setminus\mathcal{B}') - p^*_\theta(\hat{x}_{1:K})\right) \label{eq:neg_terms}
\end{align}
The other two terms in \cref{eq:three_terms} must sum to a positive value with as large or larger magnitude to \cref{eq:neg_terms} for $\Delta_\text{Var}\geq 0$. Looking at the second term, we see that
\begin{align*}
q_\mathcal{B}(\hat{x}_{1:K})\omega_{\mathcal{B}}(\hat{x}_{1:K})^2 & = \frac{p^*_\theta(\hat{x}_{1:K})^2}{q_\mathcal{B}(\hat{x}_{1:K})} \\
& = p^*_\theta(\hat{x}_{1:K})(1-q(X_{1:K}\in\mathcal{B}))\prod_{k=1}^K p^*_\theta(X_k\in\mathcal{V}_k|X_{<k}=\hat{x}_{<k}) \geq 0.
\end{align*}
This inequality becomes strict should $p^*_\theta(\hat{x}_{1:K})>0$. We will now look at the final summation in \cref{eq:three_terms}:
\begin{align*}
& \sum_{x_{1:K}\in\mathcal{Q}\setminus\mathcal{B}'}\left[q_{\mathcal{B}}(x_{1:K})\omega_\mathcal{B}(x_{1:K})^2- q_{\mathcal{B}'}(x_{1:K})\omega_{\mathcal{B}'}(x_{1:K})^2\right] \\
& = \sum_{x_{1:K}\in\mathcal{Q}\setminus\mathcal{B}'}\left(\frac{p^*_\theta(x_{1:K})^2}{q_{\mathcal{B}}(x_{1:K})} - \frac{p^*_\theta(x_{1:K})^2}{q_{\mathcal{B}'}(x_{1:K})}\right) \\
& = \sum_{x_{1:K}\in\mathcal{Q}\setminus\mathcal{B}'}p^*_\theta(x_{1:K})\left(q(X_{1:K}\in\mathcal{B}') - q(X_{1:K}\in\mathcal{B})\right)\prod_{k=1}^K p^*_\theta(X_k\in\mathcal{V}_k|X_{<k}=x_{<k}) \\
& = q(\hat{x}_{1:K})\sum_{x_{1:K}\in\mathcal{Q}\setminus\mathcal{B}'}p^*_\theta(x_{1:K})\prod_{k=1}^K p^*_\theta(X_k\in\mathcal{V}_k|X_{<k}=x_{<k})\\
& = \frac{p^*_\theta(\hat{x}_{1:K})}{\prod_{k=1}^K p^*_\theta(X_k\in\mathcal{V}_k|X_{<k}=\hat{x}_{<k})}\sum_{x_{1:K}\in\mathcal{Q}\setminus\mathcal{B}'}p^*_\theta(x_{1:K})\prod_{k=1}^K p^*_\theta(X_k\in\mathcal{V}_k|X_{<k}=x_{<k}).
\end{align*}
Since all terms in \cref{eq:three_terms} have a common factor $p^*_\theta(\hat{x}_{1:K})$, we can see that $\Delta_\text{Var}\geq 0$ iff the following holds:
\begin{align}
& 2p^*_\theta(X_{1:K}\in\mathcal{Q}\setminus\mathcal{B}') -\frac{1}{\rho(\hat{x}_{1:K})}\sum_{x_{1:K}\in\mathcal{Q}\setminus\mathcal{B}'}p^*_\theta(x_{1:K})\rho(x_{1:K}) \nonumber \\
& \quad\quad \leq (1-q(X_{1:K}\in\mathcal{B}))\rho(\hat{x}_{1:K}) + p^*_\theta(\hat{x}_{1:K}) \label{eq:final_ineq}
\end{align}
for $\rho(x_{1:K})=\prod_{k=1}^Kp^*_\theta(X_k\in\mathcal{V}_k|X_{<k}=x_{<k})$. Should this hold true, then by taking $\mathcal{B}'$ instead of $\mathcal{B}$ during the beam search segment of the hybrid approach would the variance of the sampling subroutine reduce. Generalizing this further, it is not guaranteed that the hybrid estimate will have a lower variance than regular importance sampling; however, our experimental results across a variety of settings (see Fig. 3 in the main paper) seem to indicate that the variance is reduced on average.

All of the terms to the left of the inequality in \cref{eq:final_ineq} are quantities that would require either expansive computations or estimation in order to know their values. Conversely, all the values to the right of the inequality are readily available as a byproduct of beam search. Incorporating this into decision making for our hybrid method is left for future work.

\section{Proof for Coverage-Based Beam Search Error Upper Bound (Section 4.2 in paper)} \label{sec:bs_upper_proof}

\begin{theorem}
For a given set of decoded sequences $\mathcal{B}\subset\mathcal{Q}$ with coverage $q(X_{1:K}\in\mathcal{B})$, the error between the true probabilistic query value $p^*_\theta(X_{1:K}\in\mathcal{Q})$ and the lower bound $p_\theta^*(X_{1:K}\in\mathcal{B})$ is bounded above by the complement of the coverage: $1-q(X_{1:K}\in\mathcal{B})$.
\end{theorem}

\begin{proof}
For brevity, we will assume that $\mathcal{Q}=\mathcal{V}_1\times\dots\times\mathcal{V}_K$ (although this can easily be extended to the general case). Since $q(X_k=x_k|X_{<k}=x_{<k})=\frac{p^*_\theta(X_k=x_k|X_{<k}=x_{<k})\ind(x_k\in\mathcal{V}_k)}{p^*_\theta(X_k\in\mathcal{V}_k|X_{<k}=x_{<k})}$ and $p^*_\theta(\cdot) \leq 1$, it follows that $q(X_k=x_k|X_{<k}=x_{<k}) \geq p^*_\theta(X_k=x_k|X_{<k}=x_{<k})$. This becomes a strict inequality should $p^*_\theta(X_k\in\mathcal{V}_k|X_{<k}=x_{<k}) < 1$ (which is often the case should $\mathcal{V}_k \subset \V$). Since this holds for arbitrary $k$ and $x_k$, it then follows that $q(X_{1:K}=x_{1:K}) \geq p^*_\theta(X_{1:K}=x_{1:K})$  for any $x_{1:K} \in \mathcal{Q}$.\footnote{This can be seen by comparing their autoregressive factorizations term by term.} This inequality becomes strict should any sequence not in the query set have non-zero probability, i.e., $p^*_\theta(X_{1:K}=x_{1:K})>0$ for any $x_{1:K}\in\V^K\setminus\mathcal{Q}$. 

The target value that we are estimating can be broken up into the following terms:
\begin{align*}
p^*_\theta(X_{1:K}\in\mathcal{Q}) & = p^*_\theta(X_{1:K}\in\mathcal{B}) + p^*_\theta(X_{1:K}\in\mathcal{Q}\setminus\mathcal{B}).
\end{align*}
Rearranging these terms yields us the error of our lower bound.
We can easily derive an upper bound on the error for an arbitrary set $\mathcal{B}$:
\begin{align*}
 p^*_\theta(X_{1:K}\in\mathcal{Q}) - p^*_\theta(X_{1:K}\in\mathcal{B}) &= p^*_\theta(X_{1:K}\in\mathcal{Q}\setminus\mathcal{B}) \\
 & = \sum_{x_{1:K}\in\mathcal{Q}\setminus\mathcal{B}} p_\theta^*(X_{1:K}=x_{1:K}) \\
& \leq \sum_{x_{1:K}\in\mathcal{Q}\setminus\mathcal{B}} q(X_{1:K}=x_{1:K}) \\
& = q(X_{1:K}\in\mathcal{Q}\setminus\mathcal{B}) \\
& = q(X_{1:K}\in\mathcal{Q}) - q(X_{1:K}\in\mathcal{B}) \\
& = 1-q(X_{1:K}\in\mathcal{B}) \leq 1-\alpha
\end{align*}
where $\alpha$ is the targeted coverage probability used to find $\mathcal{B}$ with coverage-based beam search.\footnote{Note that $q(X_{1:K}\in\mathcal{Q})=1$ because by design $q(X_{1:K}=x_{1:K})=0$ for every $x_{1:K}\notin\mathcal{Q}$.}
This inequality becomes strict should any sequence not in the query set have non-zero probability, i.e., $p^*_\theta(X_{1:K}=x_{1:K})>0$ for any $x_{1:K}\in\V^K\setminus\mathcal{Q}$.
\end{proof}

\section{Experimental Setup and Preparation (Section 5 in paper)}
\label{sec:dataset_model_prep}
In this section we disclose all dataset preparation and modeling details necessary for reproducing our experimental results. 

\subsection{Datasets}
We evaluate our query estimation methods on three user behavior and two language datasets. We provide details on the preparation and utilization of each below. For all datasets, users are associated with anonymous aliases to remove personally identifiable information (PII). 

\paragraph{Reviews} \citep{amazon-rev-ni-etal-2019-justifying} contains sequences of 233 million timestamped Amazon product reviews spanning from May 1996 to October 2018, with each product belonging to one of 30 product categories.  We restrict our consideration of this dataset to reviews generated by users with at least 15 product reviews and products with a defined category, retaining 63 million reviews on which the model was trained. This dataset is publicly available under the Amazon.com Conditions of Use License. 

\paragraph{Mobile Apps}%
\citep{app_dataset} consists of 3.6 million app usage records  from 200 Android users from September 2017 to May 2018, where each event is an interaction of an individual with an application. User behavior spans 87 unique applications, which we use as the vocabulary for   events   for our experiments. This dataset is released under the Creative Commons License, and all users contain at least 15 mobile app interactions so no data was removed before training. 

\paragraph{MOOCs}%
\citep{kumar2019predicting} is a dataset of sequences of  anonymized user interactions with online course materials from a set of massive open online courses (MOOCs). In total, the dataset includes 97 unique types of interactions. Data from users with fewer than 15 interaction events are not considered, resulting in a dataset of 72\% of users and 93\% of the events (350,000 interactions) of the original dataset. The MOOCs dataset is available under the MIT License.

 \paragraph{Shakespeare}%
 We also examine character-level language models, using the complete works of William Shakespeare \citep{shakespeare_data}, comprising 125,000 lines of text and 67 unique characters and released under the Project Gutenberg License. 

 \paragraph{WikiText}%
 The WikiText-v2 dataset \citep{wikitext} includes word-level language data from "verified Good" and featured articles of Wikipedia. All sentences are provided in English and the dataset is available under the Creative Commons Attribution-ShareAlike License.

\subsection{Neural Sequence Models}
For all datasets except WikiText (word-level language), we trained a 2-layer LSTM with a dropout rate of 0.3 and the ReLU activation function. Each network was trained against cross entropy loss with the Adaptive Moments (Adam) optimizer initialized with a learning rate of 0.001. A constant learning rate decay schedule and 0.01 warm-up iteration percentage was also used. All LSTM models maintain a hidden state size of 512 except the model for Shakespeare, which possesses a reduced hidden state size of 128 to reflect the size of the dataset. Model checkpoints were collected for all models every 2 epochs, and the checkpoint with the highest validation accuracy was selected to be used for query estimation experiments. All models were trained on NVIDIA GeForce 2080ti GPUs.

For WikiText-v2, we leveraged the GPT-2 medium (350 million parameters) Architecture from HuggingFace with pre-trained weights provided by OpenAI. The WikiText-v2 dataset was preprocessed using the tokenization scheme provided by HuggingFace for GPT-2, assigning numeric token indices to work pieces. No finetuning of GPT-2 is conducted.

\subsection{Ground Truth Calculations}
\label{sub:ground_truth_calc}
Exact computation of ground truth is intractable for queries with large path spaces. We circumvent this issue via the law of large numbers by computing \textbf{surrogate ground truth} query estimates with a large computational budget, leveraging the variance of the query's samples as a convergence criterion. Specifically, our algorithm is conducted as follows. We first specify a \textit{minimum} number of samples $S_{\text{low}}=10000$ to be drawn for the surrogate ground-truth estimate. Once $S_{\text{low}}$ samples have been drawn, we compute the variance of our estimate
$\hat{p}^{*}_\theta(X_{1:K}\in\mathcal{Q}) := \frac{1}{S} \sum_{i=1}^S \frac{p^{*}_\theta(X_{1:K} = x^{(i)}_{1:K})}{q(X_{1:K}=x_{1:K}^{(i)})}$ for $x^{(1)}_{1:K},\dots,x^{(S)}_{1:K} \overset{iid}{\sim} q(X_{1:K})$:
\begin{align*}
\widehat{\var}_{q}\left[\hat{p}_\theta^*(X_{1:K}\in\mathcal{Q})\right] & = \frac{1}{S}\sum_{i=1}^S \left(\frac{p^{*}_\theta(X_{1:K} = x^{(i)}_{1:K})}{q(X_{1:K}=x_{1:K}^{(i)})} - \hat{p}^*_\theta(X_{1:K}\in\mathcal{Q})\right)^2
\end{align*}

We then evaluate $\widehat{\var}_{q}\left[\hat{p}_\theta^*(X_{1:K}\in\mathcal{Q})\right]$ every 1000 additional samples until either it drops below tolerance $\delta=1e^{-7}$ or $S$ meets our maximum sample budget $S_{\text{high}}=100000$. This procedure is done in all of our experiments in which a method's performance is being compared to a query's ground truth value and exact ground truth cannot be computed due to resource constraints (typically when $K>4$).

\subsection{Model Budget Determination}
Each of the estimation methods that have been discussed can use varying amounts of computation depending on their configuration. 
To ensure an even comparison between methods, we  configure them throughout our experiments to have roughly the same amount of computation. We measure amount of computation by the number of \textbf{model calls} $f_\theta(h_k)$ used within a given method. This can be controlled directly for all of the methods except the hybrid approach. For this reason, in our experiments we typically use fix the number of samples $S$ for the importance sampling component of the hybrid method, where $S$ is the number of samples used by the hybrid method after conducting tail-splitting beam search. We then use the resulting number of total model calls used by the hybrid method (including both beam search and sampling with $S$ samples) to determine a fixed computation budget (number of model calls to compute $f_\theta(h_k)$) for all other methods. 
Should the hybrid method not be used in an experiment, the budget is set by determining the number of model calls used in drawing $S$ samples for importance sampling. 

\section{Additional Experimental Details and Results (Section 5 in paper)}
\label{sec:additional_exps}
This section  discusses additional experimental details and results that were not included in section 5 of our main paper due to space constraints. In all experiments, all means and medians reported are with respect to  
$N_Q = 1000$ randomly selected sequence locations/histories/queries per datapoint in each plot, unless stated otherwise. For each randomly selected current location, the event $a$ used in the query for $k$ steps ahead corresponds to the actual observed event $a$ for $k$ steps ahead.  




\begin{figure}
    \centering
    \includegraphics[width=1.0\textwidth]{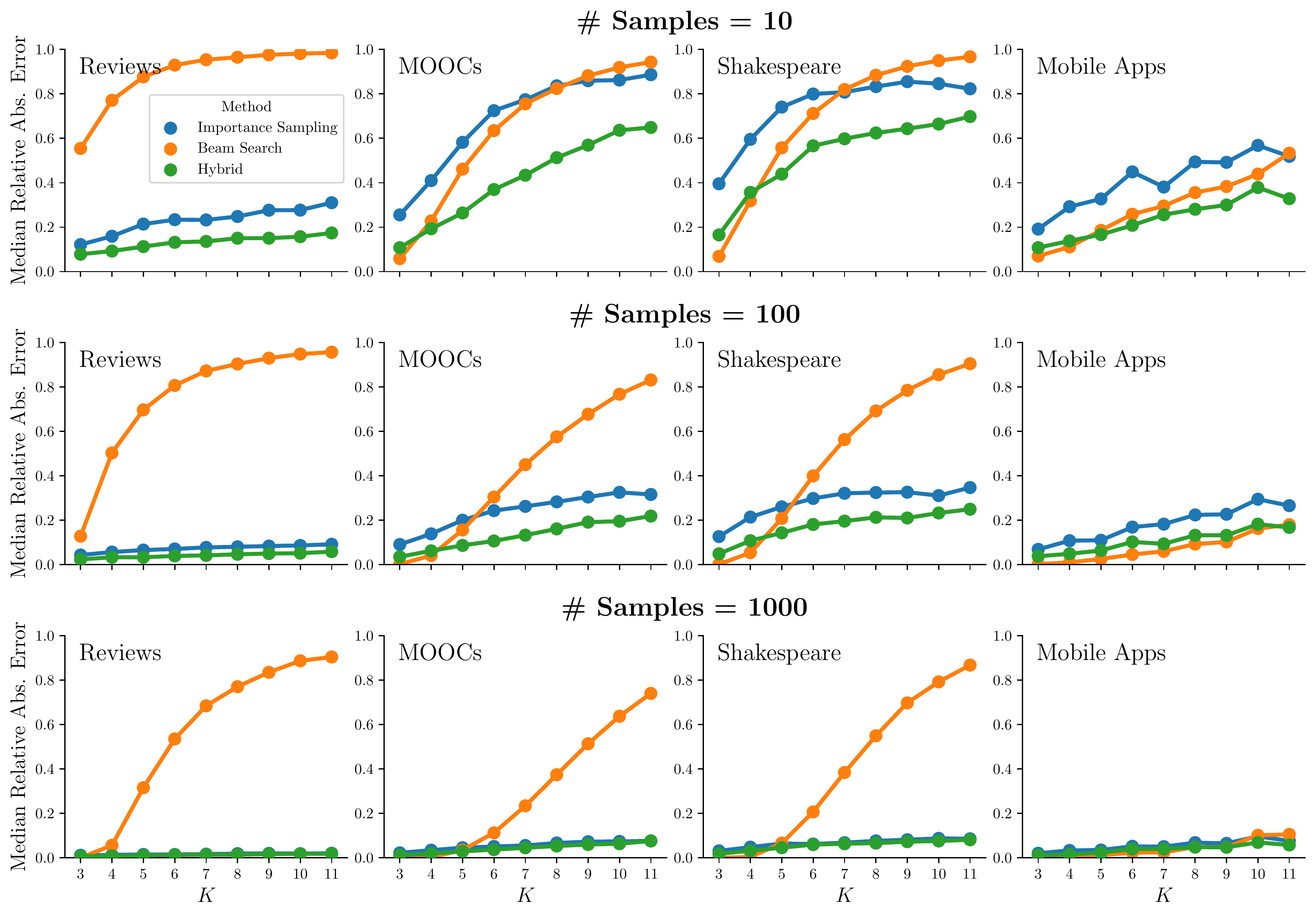}
    \caption{\textbf{Median estimation error versus query horizon $K$}: median relative absolute error (RAE) between estimated probability and (surrogate) ground truth for $p^*_\theta(\tau(\cdot)=K)$, for importance sampling, beam search, and the hybrid method, with varying computation budgets determined by 10, 100, and 1000 samples for the hybrid method. Ground truth values used to determine error in these plots are exact for $K \leq 4$ and approximated otherwise.}. 
    \label{fig:med_err_all_samples}
\end{figure}

\begin{figure}
    \centering
    \includegraphics[width=1.0\textwidth]{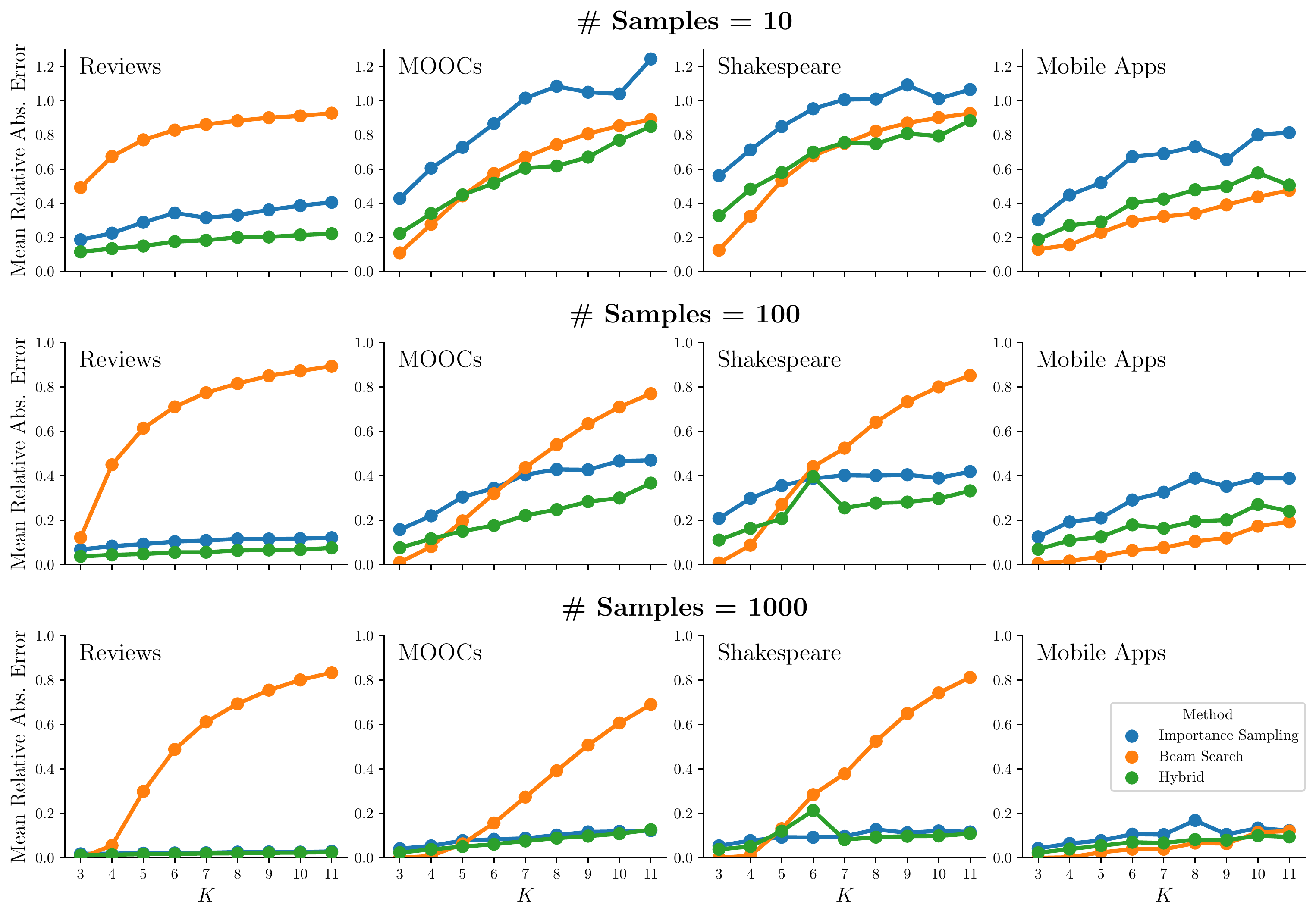}
    \caption{\textbf{Mean estimation error versus query horizon $K$}:   same format as   \cref{fig:med_err_all_samples}. Ground truth values used to determine error in these plots are exact for $K \leq 4$ and approximated otherwise.}
    \label{fig:mean_err_all_samples}
\end{figure}

\subsection{Query Estimation Error as a Function of Horizon $K$}
 Fig. 3 in the main paper shows the median relative absolute error (RAE) (across $N_Q=1000$ queries) as a function of query horizon $K$ for 4 datasets, comparing beam search, importance sampling, and the hybrid method, with a computation budget fixed   at $S=100$ hybrid samples. Here we provide a number of extensions of these results.

\cref{fig:med_err_all_samples} shows the median   RAE, for three levels of computation budget: $S=10, 100, 1000$ and \cref{fig:mean_err_all_samples} shows the same results but now reporting mean RAE on the y-axis. While the details differ across different settings, the qualitative conclusions in these Figures agree with those for Fig. 3 in the main paper, namely that beam search is more sensitive (in its error) to both the horizon query $K$ and to individual datasets, compared to both importance sampling and the hybrid method. More granular perspectives of this information for one of the datasets (MOOCs) can be seen in \Crefrange{fig:err_scatter_10}{fig:err_scatter_1000} in the form of scatter plots of each of the individual query estimates against (surrogate) ground truth for that query. Different budgets are shown in different figures, and the results for $K=3, 5, 7, 9, 11$ are shown in each column.

\begin{figure}
    \centering
    \includegraphics[width=1.0\textwidth]{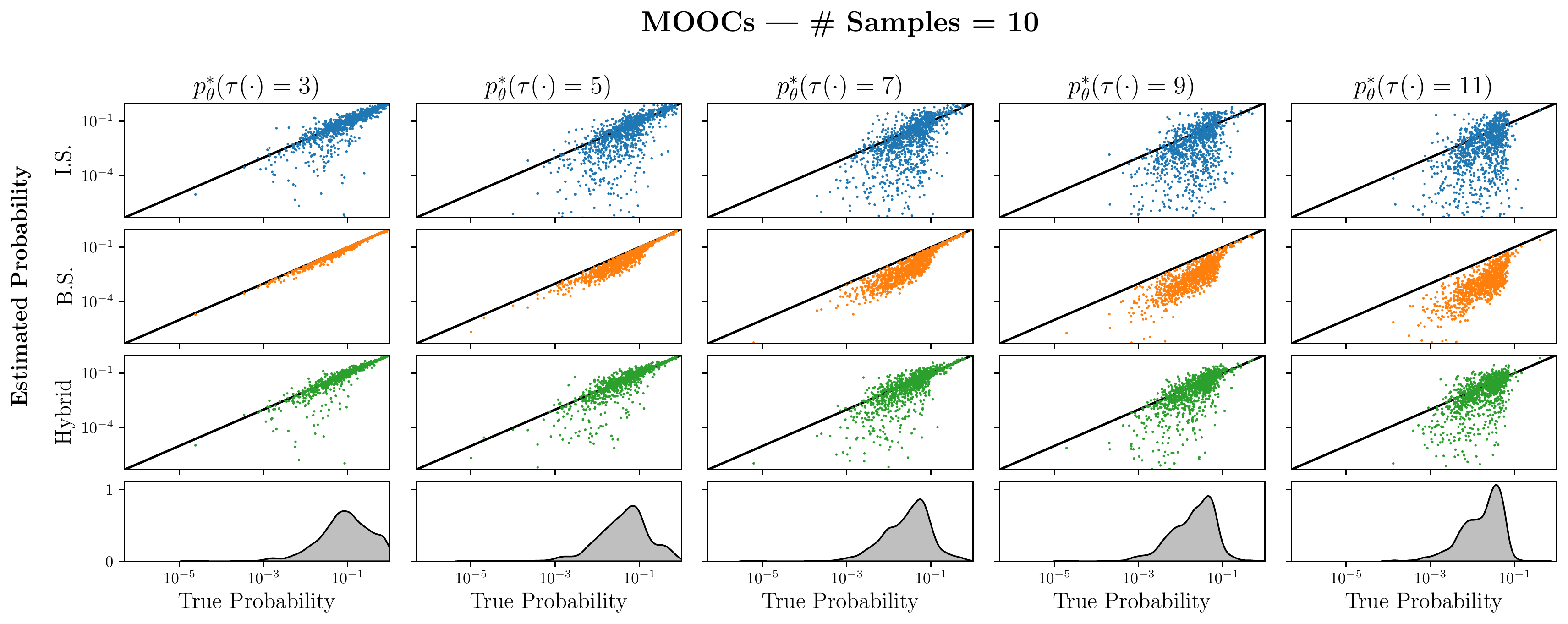}
    \caption{\textbf{Scatterplots of individual query estimates versus (surrogate) ground truth, computation budget of 10 hybrid samples:}  Comparison of importance sampling (I.S.), beam search (B.S.), and the hybrid method for the MOOCs dataset with the budget determined by the hybrid method using 10 samples. The x-axis corresponds to the surrogate ground truth values for a given query result. Density plots at the bottom are for the surrogate ground truth values. Ground truth values used to determine error in these plots are exact for $K \leq 4$ and approximated otherwise.}
    \label{fig:err_scatter_10}
\end{figure}

\begin{figure}
    \centering
    \includegraphics[width=1.0\textwidth]{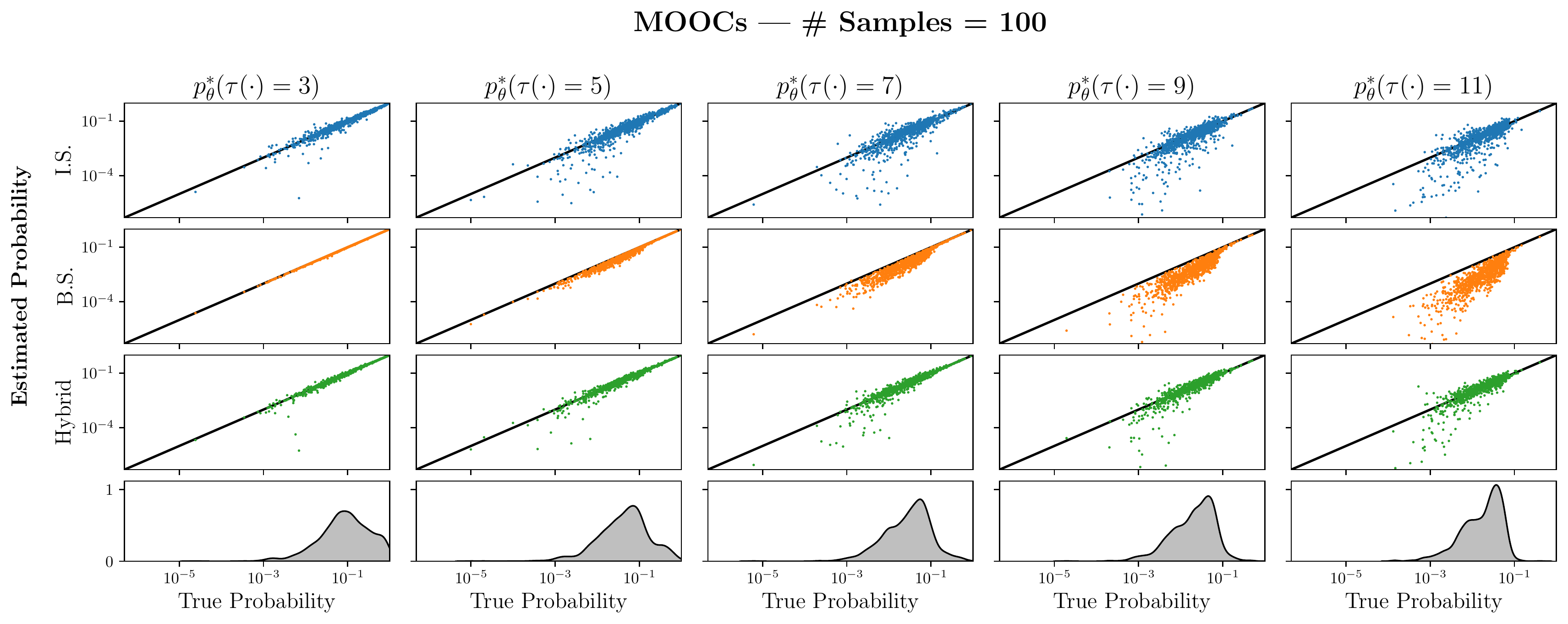}
    \caption{\textbf{Scatterplots of individual query estimates versus (surrogate) ground truth, computation budget of 100 hybrid samples:} Same format as \cref{fig:err_scatter_10}.}
    \label{fig:err_scatter_100}
\end{figure}
\begin{figure}
    \centering
    \includegraphics[width=1.0\textwidth]{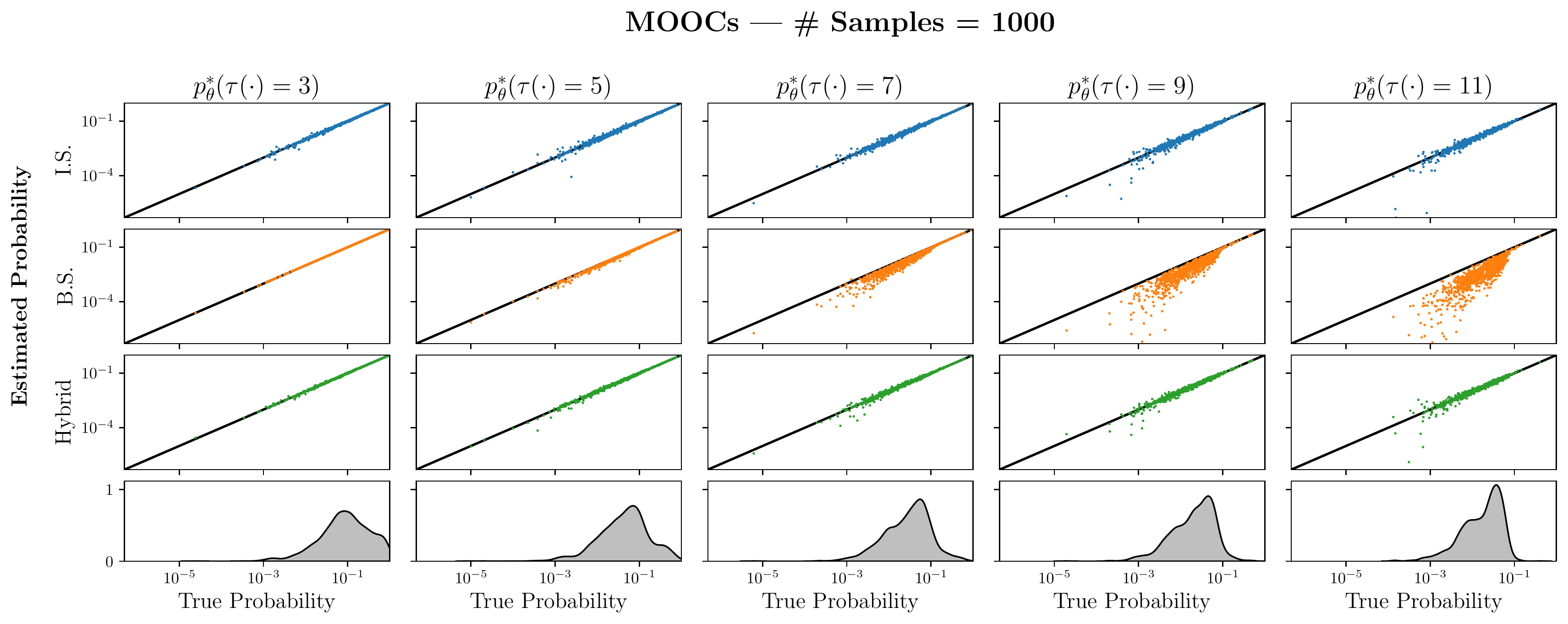}
    \caption{\textbf{Scatterplots of individual query estimates versus (surrogate) ground truth, computation budget of 1000 hybrid samples:} Same format as \cref{fig:err_scatter_10}.}
    \label{fig:err_scatter_1000}
\end{figure}

\subsection{Query Estimation Error as a Function of Computation Budget}

In addition to identifying optimal query estimation methodologies for a low and fixed computation budget, we also explore the impact of increasing computation budget for query lengths $K=3,7,11$, roughly corresponding to short, medium, and long horizon queries. These experiments are conducted in the same manner as the query estimation experiments with a fixed model budget, but are then repeated for many different budgets derived from $S=10,30,50,100,300,500,1000,3000,5000,10000$ hybrid samples. The intention with these experiments is to observe if any query estimation methods disproportionately benefit from increased computation and exhibit behavior that was not present at lower computation budgets. We also include two additional baselines. The full set of methods explored is listed below.
\begin{enumerate}
  \item Importance sampling (informative proposal distribution $q$ derived from model $p_\theta$)
  \item Beam search 
  \item Hybrid search and sampling 
  \item Monte-Carlo sampling with a uniform proposal distribution
  \item Naive model sampling (direct MC sampling for $\E_{p_\theta^*} \ind(X_{1:K}\in\mathcal{Q})$)
\end{enumerate}

As a clarifying point, naive sampling is conducted by sampling sequences from the model and determining if they fall within the query set $\mathcal{Q}$; the proportion of samples that exist in $\mathcal{Q}$ serves as a naive means of determining the query estimate in question. In addition, Monte-Carlo sampling with a uniform proposal distribution samples sequences in $\mathcal{Q}$ uniformly and the estimates the query probability for that sample. These two methods were not included in the main paper due to their consistently poor performance.

\begin{figure}
    \centering
    \includegraphics[width=1.0\textwidth]{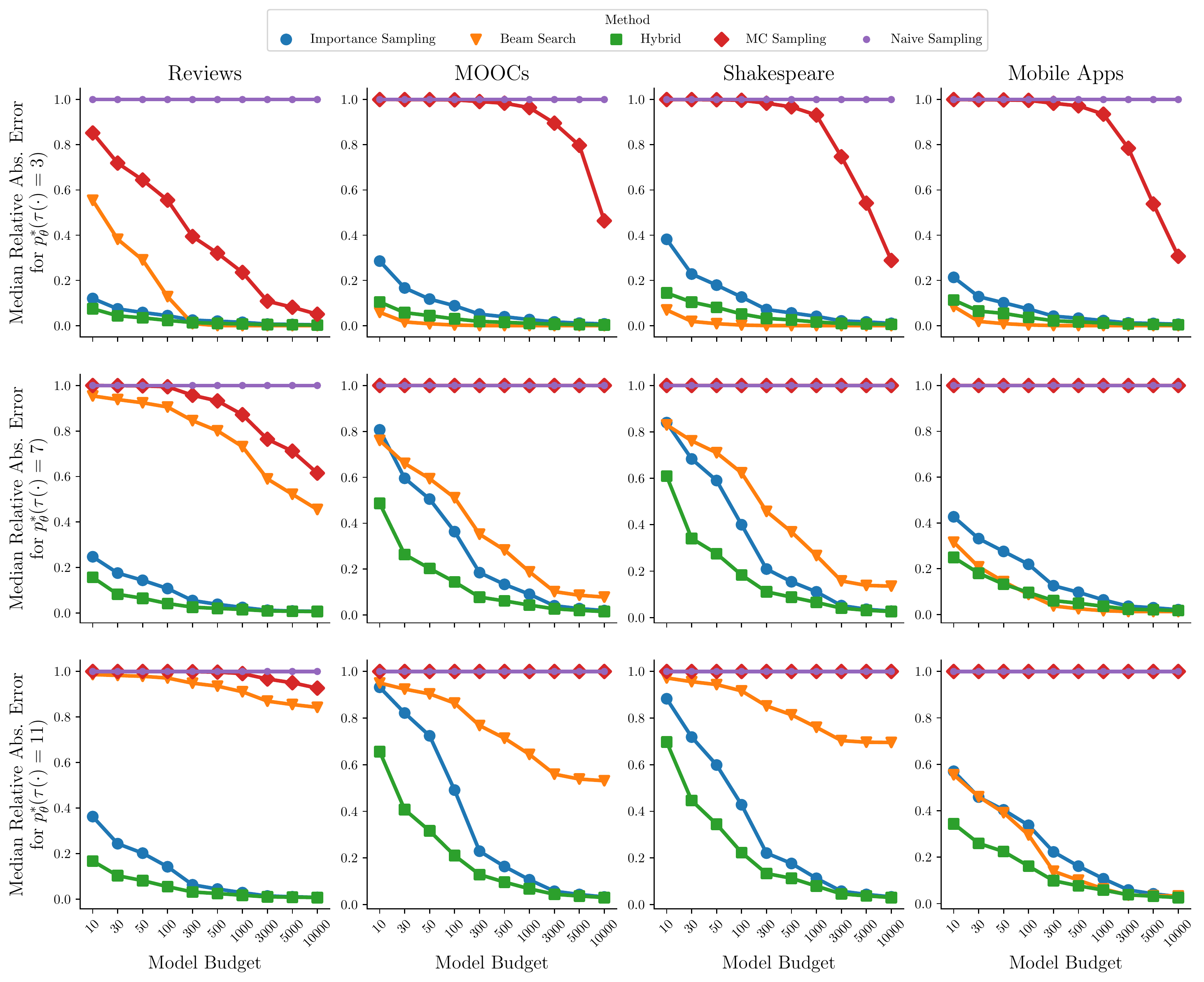}
    \caption{\textbf{Median estimation error versus model budgets:}  Median relative absolute error (RAE) for importance sampling, beam search, the hybrid method, MC sampling, and naive sampling across three different queries, $p^*_\theta(\tau(\cdot)=k)$, for $k=3,7,11$, over all four main datasets, as a function of different model budgets. For cases where the MC sampling results are not visible (e.g., see rightmost plot on the bottom row), the results coincide with the naive sampling results. Ground truth values used to determine error in these plots are exact for $K \leq 4$ and approximated otherwise.}
    \label{fig:budget}
\end{figure}

\begin{figure}
    \centering
    \includegraphics[width=1.0\textwidth]{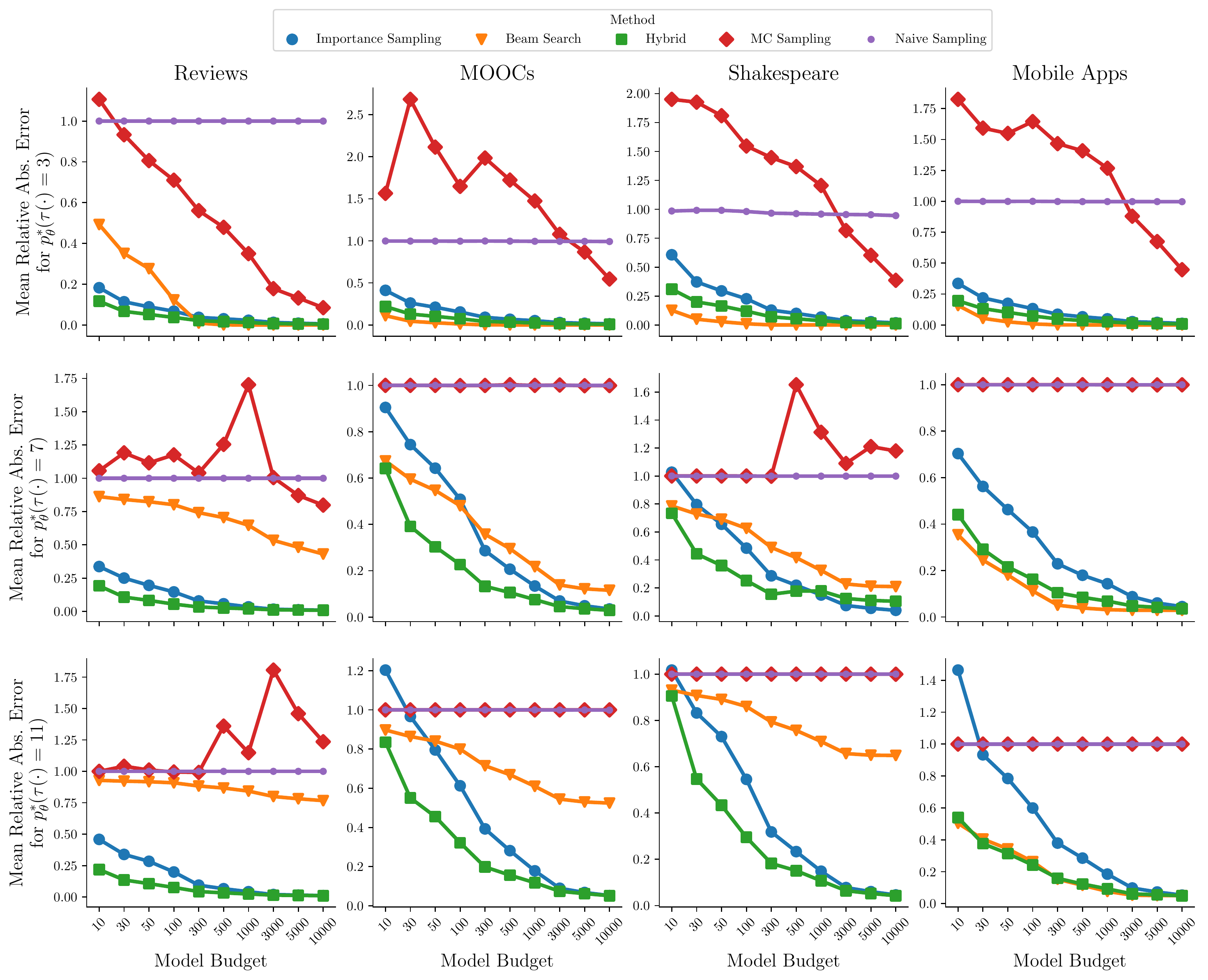}
    \caption{\textbf{Mean estimation error versus model budgets:} Same format as \cref{fig:budget}. Ground truth values used to determine error in these plots are exact for $K \leq 4$ and approximated otherwise.}
    \label{fig:med_budget}
\end{figure}

\cref{fig:budget} (median error) and \cref{fig:med_budget} (mean error), show that increasing the computation budget by an order of magnitude roughly corresponds to a three times reduction in RAE for both importance sampling and the hybrid method.  Naive sampling sees almost no benefit from the increased budget regardless of the size of the query path space. Monte Carlo sampling sees some reduction in error from  increased computation budget for some configurations, but also often sees no benefit. 

For the provided budgets, the query estimates resulting from naive model sampling are consistently 0, resulting in an RAE of 1. While naive model sampling can be useful in some contexts in general,  these results indicate that it is not well-suited for estimating query probabilities. This is likely because many of the ground truth probabilities for these queries have values on the order of $10^{-1}$ or smaller. 
For queries that are highly unlikely under the model, the probability of even a single sampled sequence belonging to $\mathcal{Q}$ is very low. 

Monte-Carlo sampling also includes high error estimates, but for a different reason. Since the Monte-Carlo estimate can be decomposed into an expectation over $p(X_{1:K}=x_{1:K}, x_{1:K} \in \mathcal{Q})$ that is then re-scaled by $|\mathcal{Q}| \gg 0$, the scaling term magnifies any error in the expectation dramatically, inducing extremely high variance and the potential to produce query estimates that exceed 1. This high variance can persist even for high computation budgets. By contrast, beam search improves with an increased computational budget, but only as a function of the total path space. The larger the path space and the higher the entropy of the distribution, the worse the beam search estimates are as measured by RAE.

\begin{figure}
    \centering
    \includegraphics[width=\textwidth]{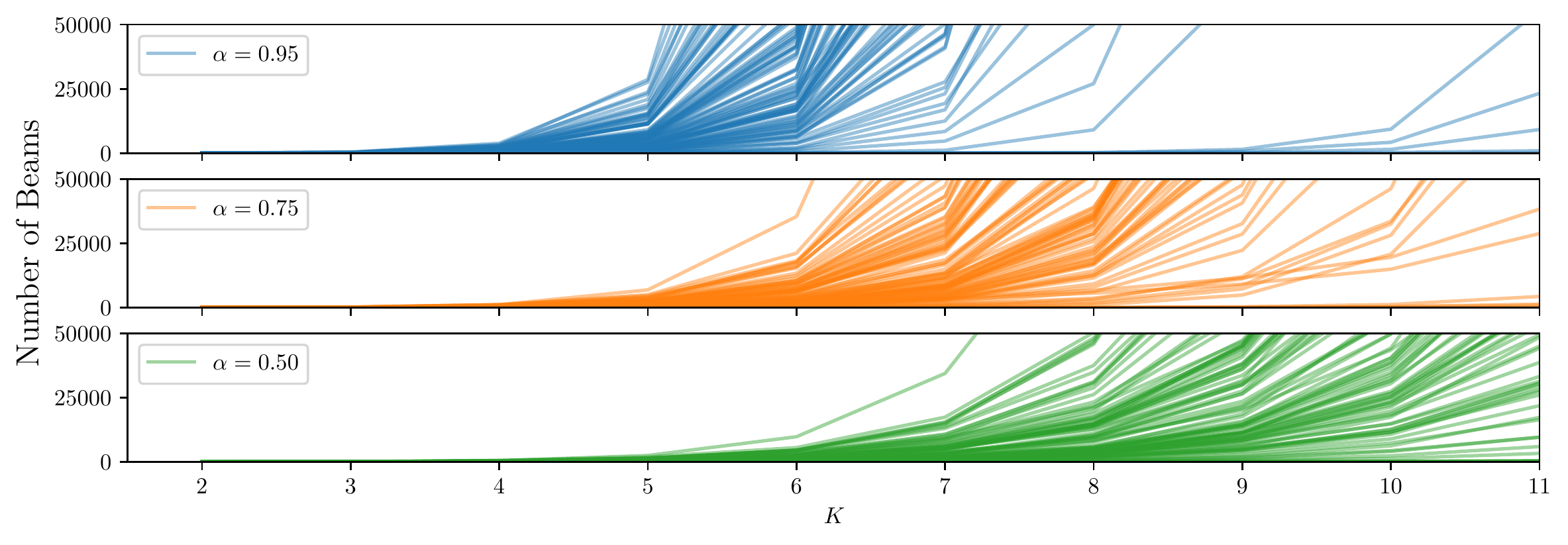}
    \caption{Number of beams as a function of $K$ steps into decoding sequences for coverage-based beam search estimating $p^*_\theta(X_K)$ with $\alpha\in\{0.95, 0.75, 0.5\}$ over a variety of different starting histories in the Shakespeare dataset.}
    \label{fig:coverage}
\end{figure}

\subsection{Coverage-based Beam Search Ablation}%
\label{sub:coverage_bs_ablation}
As described in the main paper, a variant of beam search well-suited to query estimation is \textit{coverage-based beam search}. However, this method was not explored further in our analysis as it could not scale to non-trivial query types and large query path spaces. Depicted in \cref{fig:coverage}, the minimum number of beams needed to cover 50\%, 75\%, and 90\% of the query path space increases exponentially with the query length $K$. Though coverage-based beam search comes with desirable coverage guarantees and is a consistent (albeit biased) query estimator, naively applying coverage-based beam search to queries of practical interest is ill-posed and computationally intractable. However, the hybrid method preserves some of beam search's desirable traits while making it suitable for queries of practical interest. See \cref{sec:hybrid_details} for more information.

\subsection{Long-horizon Query Estimation}

The empirical results of our experiments tell us much of the abilit, but we still do not have clarity on the extent to which importance sampling can yield low-error estimates in the limit of large sequence lengths. To that end, we conducted an experiment to directly assess the performance of importance sampling when estimating long horizons. Our primary conclusion was that (i) with reasonable sample sizes the error remains below 25\%, even at $K=100$, and (ii) the error does not increase substantially as $K$ increases beyond $K=5$.

In more detail, a set of 100 history sequences $\mathcal{H}$ of length 5 was collected from each dataset. We then compute ground truth or pseudo ground truth (PGT) for query 2 (probability of event at $K$ without restriction) over sequence lengths $K =[2, 5, 10, 20, 40, ..., 100]$. Pseudo-ground truth estimates are generated with a tolerance $\delta=1e^{-7}$ and a maximum sampling budget of $250,000$. The same estimates are then computed using importance sampling with sampling budgets $[10, …, 10000]$. Similar to the experiments in Section \cref{sec:experiments}, we aggregate and present the results as the median relative absolute error, shown in the tables below. Since we are not focusing on a specific event type  of interest, this median also marginalizes over all event types  in the vocabulary as well as the sampled query sequences. All results are reported as percentages

\begin{table}[!ht]
    \centering
    \begin{tabular}{l|cccccccc}
    \toprule
        \# Samples & 2 & 5 & 10 & 20 & 40 & 60 & 80 & 100 \\
        \midrule
        10 & 31.35 & 40.33 & 44.39 & 43.48 & 45.98 & 45.84 & 49.44 & 50.05 \\ 
        100 & 13.29 & 17.12 & 18.73 & 19.46 & 20.41 & 21.65 & 21.82 & 21.86 \\ 
        1000 & 4.94 & 6.38 & 7.06 & 7.33 & 7.26 & 6.99 & 7.13 & 6.82 \\ 
        10000 & 1.60 & 2.05 & 2.21 & 2.34 & 2.20 & 2.13 & 2.06 & 2.10 \\ 
    \bottomrule
    \end{tabular}
    \caption{Median RAE across query estimations methods (1000 samples) and query horizons $K=2,5,10,20, ..., 100$ and sample budgets $10, 100, 1000, 10000$ for Amazon Reviews.}
    \label{table:long_horizon_amazon_reviews}
\end{table}

\begin{table}[!ht]
    \centering
    \begin{tabular}{l|cccccccc}
    \toprule
        \# Samples & 2 & 5 & 10 & 20 & 40 & 60 & 80 & 100 \\ 
        \midrule
        10 & 66.21 & 84.93 & 92.92 & 96.96 & 98.41 & 99.00 & 99.19 & 99.32 \\ 
        100 & 62.36 & 80.94 & 89.79 & 94.00 & 96.16 & 97.05 & 97.38 & 97.66 \\ 
        1000 & 47.96 & 59.14 & 69.10 & 75.45 & 74.79 & 71.34 & 65.12 & 59.75 \\ 
        10000 & 15.21 & 20.51 & 23.38 & 24.00 & 20.66 & 18.78 & 16.48 & 15.55 \\ 
        \bottomrule
    \end{tabular}
    \caption{Median RAE across query estimations methods (1000 samples) and query horizons $K=2,5,10,20, ..., 100$ and sample budgets $10, 100, 1000, 10000$ for Mobile Apps.}
    \label{table:long_horizon_apps}
\end{table}

\begin{table}[!ht]
    \centering
    \begin{tabular}{l|cccccccc}
    \toprule
        \# Samples & 2 & 5 & 10 & 20 & 40 & 60 & 80 & 100 \\  
        \midrule
        10 & 79.84 & 79.99 & 83.53 & 85.29 & 83.53 & 84.00 & 85.82 & 85.15 \\ 
        100 & 28.66 & 32.85 & 39.25 & 40.66 & 38.78 & 41.51 & 39.71 & 40.07 \\ 
        1000 & 8.61 & 10.99 & 13.29 & 13.52 & 13.78 & 13.89 & 13.64 & 14.10 \\ 
        10000 & 2.77 & 3.48 & 4.28 & 4.34 & 4.32 & 4.33 & 4.34 & 4.30 \\ 
    \bottomrule
    \end{tabular}
    \caption{Median RAE across query estimations methods (1000 samples) and query horizons $K=2,5,10,20, ..., 100$ and sample budgets $10, 100, 1000, 10000$ for Shakespeare.}
    \label{table:long_horizon_shakespeare}
\end{table}

\begin{table}[!ht]
    \centering
    \begin{tabular}{l|cccccccc}
    \toprule
        \# Samples & 2 & 5 & 10 & 20 & 40 & 60 & 80 & 100 \\ 
        \midrule
        10 & 89.62 & 97.30 & 98.73 & 99.19 & 99.08 & 99.06 & 98.91 & 98.87 \\ 
        100 & 63.80 & 83.34 & 84.77 & 79.15 & 66.89 & 62.99 & 60.28 & 61.30 \\ 
        1000 & 30.39 & 43.37 & 36.32 & 26.07 & 19.78 & 18.11 & 17.53 & 17.38 \\ 
        10000 & 11.92 & 15.15 & 11.35 & 7.86 & 5.90 & 5.66 & 5.43 & 5.40 \\
        \bottomrule
    \end{tabular}
    \caption{Median RAE across query estimations methods (1000 samples) and query horizons $K=2,5,10,20, ..., 100$ and sample budgets $10, 100, 1000, 10000$ for MOOCs.}
    \label{table:long_horizon_moocs}
\end{table}

In general, we see (not surprisingly) that the increase in sequence length leads to a consistent and non-trivial increase in error for most sampling budgets. In addition, as expected the increase in sampling budget consistently reduces the query estimation error. However, we do witness the interesting phenomenon that the error occasionally \textit{decreases} as the sequence length increases. We conjecture that this may be happening because as the sequence length increases, the relevance of the history context $\mathcal{H}$ decreases and the distribution may regress to a base stationary distribution (as if no history context were provided at all) indicating that the conditional model entropy may be the main driving factor in estimation complexity. This intuition is further supported by the fact that in many datasets, budgets exist where query estimation error first increases but then begins to decrease again. Regardless, as the largest budget of $10,000$ we witness that median RAE remains at or under 25\% in all cases, often significantly so.

\subsection{Query Estimation with Large-Scale Language Models}%
\label{sec:gpt_exp}

In order to explore the feasibility of applying our query estimation methods to real-world sequence data, we also analyze a subset of our query estimation methods against GPT-2 and WikiText language data. GPT-2 decomposes English words into $V=50257$ work pieces, a vocabulary over 500 times larger than our other datasets. For this reason, we only conduct experiments using only 100 sequence histories per dataset due to computational limitations. For the same reason, we do not explore the hybrid method and restrict ourselves to analyzing the following query estimators:
\begin{enumerate}
  \item Importance sampling (informative proposal distribution $q$ derived from model $p_\theta$)
  \item Beam search 
  \item Monte-Carlo sampling with a uniform proposal distribution
\end{enumerate}

\begin{table}
\centering
\begin{tabular}{cccccc}
\toprule
K &  Importance Samp. &  Beam Search &  MC Samp. &  Entropy &   Entropy \% \\
\midrule
         3 &                     \textbf{11.41} &              51.25 &                  99.36 &    12.89 &        41.28 \\
        4 &                     \textbf{13.35} &              82.42 &                  99.95 &    19.09 &        40.74 \\
        5 &                      \textbf{13.53} &              93.59 &                  99.99 &    25.23 &        40.38 \\
\bottomrule
\end{tabular}
\caption{\textbf{Median RAE across query estimations methods (1000 samples) and query horizons $K=3,4,5$ for GPT-2 and Wikitext}. Entropy values estimated as the mean (over 100 queries) of the restricted proposal $q$ and entropy \% is the entropy in percentage relative to its potential maximum value $Klog(V)$.}
\label{table:wikitext_and_entropy}
\end{table}

Fixed-budget query experiments with GPT-2 are conducted identically to those on the other 4 datasets. We find query estimation error closely mirrors the results we see in datasets with smaller vocabulary sizes, suggesting our findings may generalize well to practical domains. Our analysis is reported in \cref{table:wikitext_and_entropy} and includes estimates of the restricted model entropy $H(q)$ for different query lengths $K$, with the entropy increasing much faster than small-vocabulary models, as expected. With that said, there is still much exploration to be done on large-scale sequence models and is a promising avenue for future work.

\begin{figure}
    \centering
    \includegraphics[width=\textwidth]{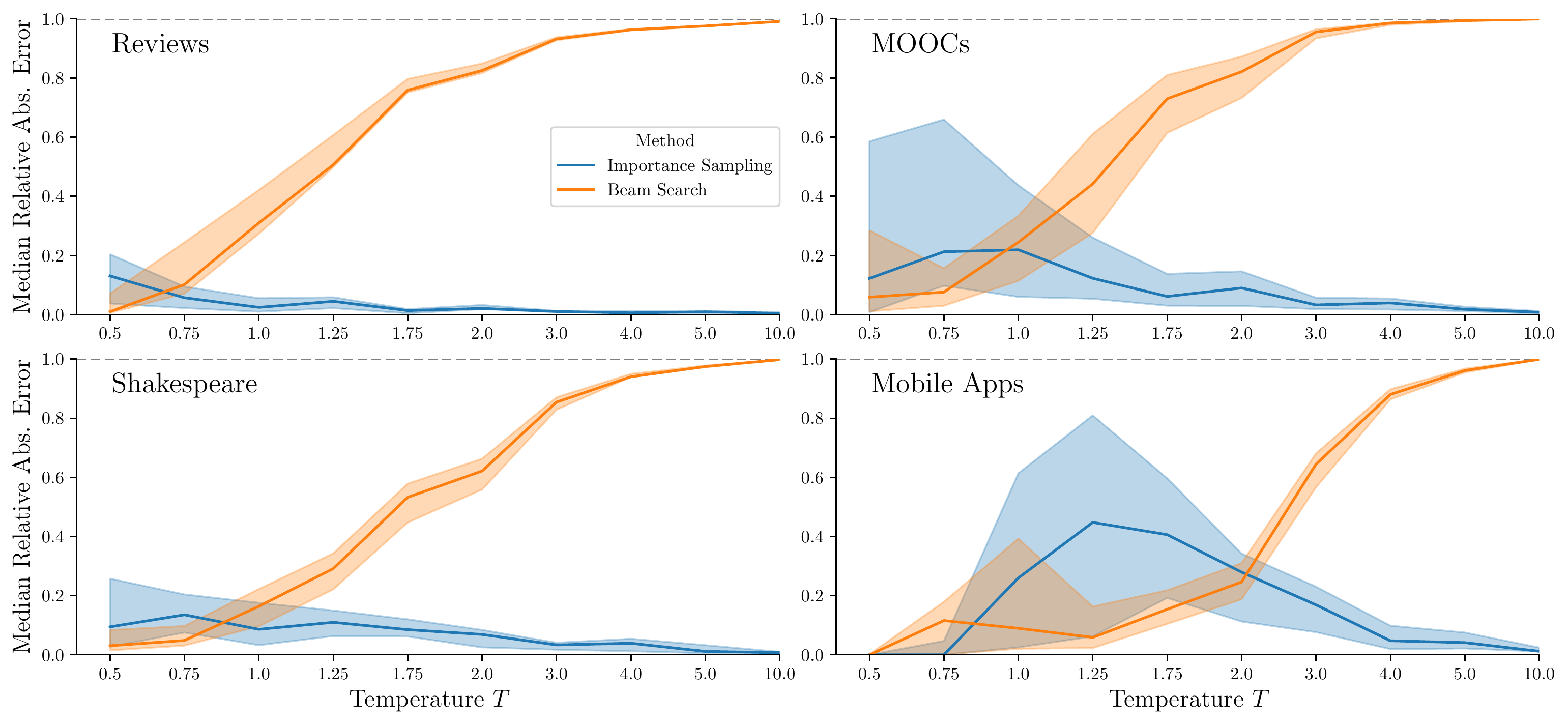}
    \caption{\textbf{Error as a function of entropy:} Median relative absolute error for estimating $p^*_{\theta,T}(\tau(\cdot)=4)$ over all four datasets for a variety of imposed temperatures $T$. Shaded regions indicate interquartile ranges.}
    \label{fig:temp_all_datasets}
\end{figure}

\subsection{Entropy Relationship with Query Estimation Error}

Fig. 4(b) in the main paper demonstrated how indirectly controlling the entropy of a given model through an applied temperature affected the performance of beam search and importance sampling. In \cref{fig:temp_all_datasets} we can see similar plots for all of four of our main datasets.

\subsection{Investigation of Query 4 (``A" before ``B")}

Most of our analysis was conducted on hitting time queries, and justifiably so as more advanced queries decompose into individual operations on hitting times. This includes Q4, colloquially stated as the probability an item from token set $A$ occurs before an item in token set $B$. More formally:
\begin{align*}
p_\theta^*(\tau(A) < \tau(B) ) & = \sum_{k=1}^\infty p_\theta^*(\tau(A)=k, \tau(B)>k) \\
& = \sum_{k=1}^\infty \sum_{a\in A}p_\theta^*(X_k=a, X_{<k}\in (\V\setminus(A\cup B))^{k-1}) 
\label{eq:hit_comp_sum}
\end{align*}
While this cannot be computed exactly, a lower bound can. The other option is to produce a lower bound on this expression by evaluating the sum in \cref{eq:hit_comp_sum} for the first $K$ terms. We can achieve error bounds on this estimate by noting that $p_\theta^*(\tau(a) < \tau(b) ) + p_\theta^*(\tau(a) > \tau(b) ) = 1$. As such, if we evaluate \cref{eq:hit_comp_sum} up to $K$ terms for both $p_\theta^*(\tau(a) < \tau(b) )$ and $p_\theta^*(\tau(b) < \tau(a) )$, the difference between the sums will be the maximum error either lower bound can have. This difference will be referred to as \emph{unaccounted probability} and will approach 0 as $K\rightarrow\infty$.

A natural question to ask is what is the minimum value of $K$ sufficient to compute these lower bounds to in order to have negligible unaccounted probability. Though this will surely vary based on the entropy of the model and the specific query in question, we explore this question across all datasets except WikiText and note some general trends.

\begin{figure}
    \centering
    \includegraphics[width=\textwidth]{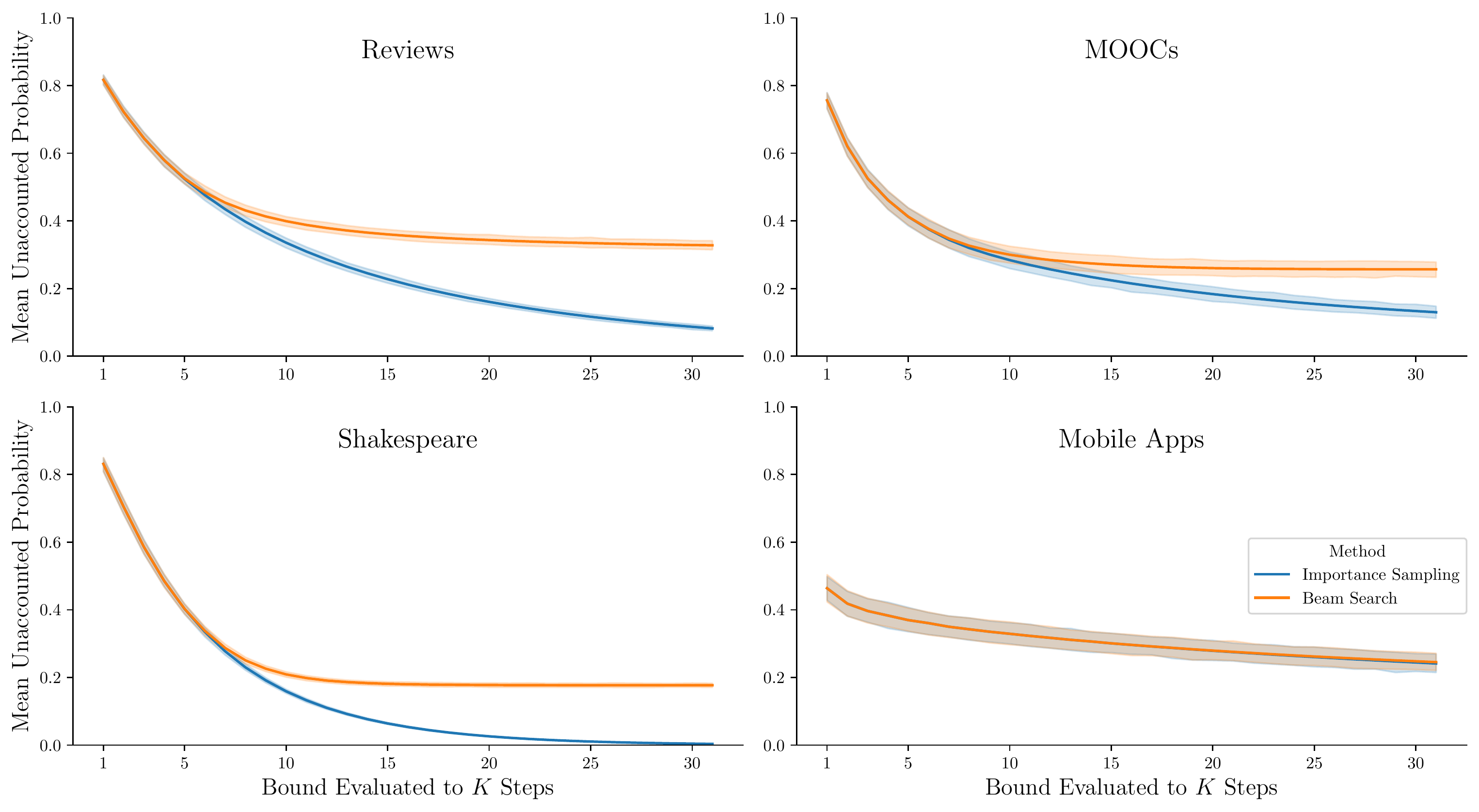}
    \caption{Mean unaccounted probability ($1-(\hat{p}^*_\theta(\tau(A)<\tau(B)) + \hat{p}^*_\theta(\tau(B)<\tau(A)))$) when evaluating the pair of queries for $\tau(A) < \tau(B)$ and $\tau(B) < \tau(A)$ up to $K$ steps into the future over all four main datasets. Shaded regions indicate 99\% confidence intervals and estimates were computed with a fixed model budget based on 1000 samples.}
    \label{fig:a_before_b_investigation}
\end{figure}

\cref{fig:a_before_b_investigation} plots the unaccounted probability $1-\hat{p}^{*}_\theta(\tau(A) < \tau(B))-\hat{p}^*_\theta(\tau(A)>\tau(B))$ as a function of query length $K$. We observe that for many datasets, a query horizon of $k=30$ is largely sufficient to reduce the remaining probability to under 10\%. One notable exception is the Mobile Apps dataset, which, due to its lower entropy and high self-transition rate, maintains a much longer query horizon. This discovery implies that a successful partition of probability space with a given Q4 query can be sensitive to the model distribution, but also that approximate partitions are possible for relatively low values of $K$. 

\subsection{Linearly Compounding Errors in Complex Query Estimation}

As mentioned in Section \cref{sec:experiments} of the main paper, hitting time queries can often be seen as components of more involved queries, such as “a” before “b” queries $p_\theta^*(\tau(a) < \tau(b))$ or counting-style queries $p_\theta^*(N_a(K)=n)$. These queries can be rewritten as summations of more basic hitting time queries. For instance, $p_\theta^*(\tau(a) < \tau(b)) = \sum_{k=1}^\infty p_\theta^*(\tau(a)=k,  \tau(b)>k)$. In practice, each summand probability is estimated using our proposed techniques so it can be seen that the error compounds additively with respect to the different basic hitting time queries. 

More generally, our framework proposes representing general queries as $\mathcal{Q} = \cup_i \mathcal{Q}_i = \cup_i \prod_j \mathcal{V}_j^{(i)}$ such that the $\mathcal{Q}_i$’s form a minimal partition on $\mathcal{Q}$. With this representation, we can see that for an arbitrary query, the error when estimating will compound additively and scale linearly with respect to the number of different $\mathcal{Q}_i$. Note that the actual values of $p_\theta^*(x\in\mathcal{Q}_i)$ do have an impact on the errors when estimating due to the values $\in[0,1]$. Lastly, as mentioned in the paper we can additionally control this error either by utilizing the coverage-based beam search, or by leveraging the Central Limit Theorem with importance sampling.

\subsection{Qualitative Exploration and Practical Applications}%

In addition to systematic quantitative analysis of query estimators, we also qualitatively explore specific applications of our methods that lend practical insights. First, we consider the question of ``given a partial sentence, predict when the sentence will end". 
Using our query estimation methods, we can not only answer this question with relatively low error but also effectively re-use intermediate computational results. 
This necessarily correlates query estimates over steps $K$, but this confers little negative impact upon the analysis since our sampling methods are unbiased estimators and beam search, though biased, offers a deterministic lower bound. The results of our analysis are seen in Fig. 1 
in the main paper and align with basic intuition about English sentences, with open ended prefixes possessing a long-tailed end-of-sentence probability distribution relative to more structured and declarative phrases.

\begin{figure}
    \centering
    \includegraphics[width=\textwidth]{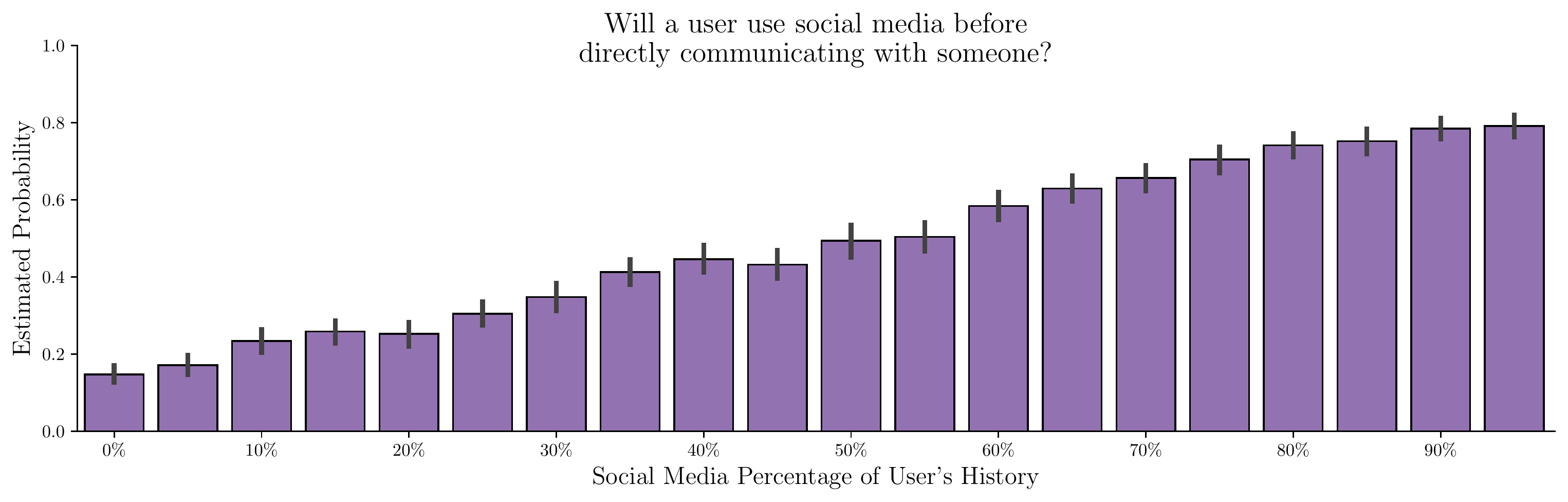}
    \caption{A case study of evaluating how likely it is that a mobile app user will use social media before they directly communicate with someone via text messaging, email, phone call, etc. as a function of how much of their recent history consists of social media usage. This is a $p^{*}_\theta(\tau(A)<\tau(B))$ type query that has been estimated and averaged over 500 different histories for each social media history percentage. Error bars indicate 90\% confidence intervals.}
    \label{fig:flashy_ab}
\end{figure}

A second practical question that can be asked with our query estimation methods utilizes the Mobile Apps dataset: given someone's mobile usage history, predict what will occur first: the individual will go on social media or directly interact with someone via video chat, call, text, or email. This question is a practical application of query 4, $p^{*}_\theta(\tau(a) < \tau(b))$. Other, equally interesting equivalents of this question include ``will an online shopper purchase something before leaving the website?". Below, we have conducted an experiment where we synthetically generate mobile app behavior histories with specific percentages of social media activity present. By computing these queries over several histories and then averaging the estimates, we obtain the results seen in \cref{fig:flashy_ab}. As expected, we see a clear linear trend between a user's social media usage and the likelihood they will return to it before conducting other tasks like directly communication. Such a result, though contrived and purely demonstrative in this setting, could be applied to many practical applications that analyze the characteristics of online user behavior.
\end{document}